\documentclass[10pt,twocolumn,letterpaper]{article}

\usepackage[pagenumbers]{cvpr} %

\usepackage{graphicx}
\usepackage{amsmath}
\usepackage{amssymb}
\usepackage{amsthm}
\usepackage{amsfonts}
\usepackage{booktabs}
\usepackage{soul}
\usepackage{xcolor}

\usepackage{xcolor}
\definecolor{pinegreen}{rgb}{0.0, 0.47, 0.44}

\usepackage[pagebackref,breaklinks,colorlinks]{hyperref}

\usepackage[capitalize]{cleveref}
\crefname{section}{Sec.}{Secs.}
\Crefname{section}{Section}{Sections}
\Crefname{table}{Table}{Tables}
\crefname{table}{Tab.}{Tabs.}

\def\cvprPaperID{9568} %
\def\confName{CVPR}
\def\confYear{2022}

\DeclareMathOperator{\enc}{enc}

\usepackage[ruled]{algorithm2e}

\newcommand{\argmin}{\mathop{\mathrm{argmin}}}

\renewcommand{\paragraph}[1]{\vspace{0.5em}\noindent\textbf{#1}\ }

\newcommand{\pub}{\text{pub}}
\newcommand{\pri}{\text{pri}}
\newcommand{\E}{\mathbb{E}}

\newtheorem{theorem}{Theorem}
\newtheorem{lemma}[theorem]{Lemma}
\newtheorem{proposition}[theorem]{Proposition}

\newtheorem{corollary}[theorem]{Corollary}
\newtheorem{definition}[theorem]{Definition}

\def\E{\mathbb{E}}

\def\R{\mathbb{R}}

\def\cD{\mathcal{D}}

\def\cL{\mathcal{L}}
\def\cM{\mathcal{M}}
\def\cN{\mathcal{N}}

\def\cR{\mathcal{R}}

\def\cW{\mathcal{W}}

\def\cY{\mathcal{Y}}
\def\cZ{\mathcal{Z}}

\SetKwComment{Comment}{// }{}

\begin{document}

\title{
Mixed Differential Privacy in Computer Vision
}

\author{Aditya Golatkar$^{*,2}$ \ \ Alessandro Achille$^1$ \ \  Yu-Xiang Wang$^{1,3}$ \\ Aaron Roth$^{1,4}$ \ \ Michael Kearns$^{1,4}$ \ \  Stefano Soatto$^{1,2}$\\
$^1$AWS AI Labs\ \ \ $^2$UCLA \ \ \ $^3$ UCSB \ \ \ $^4$ Penn\\
\thanks{work done during an Amazon internship}
\texttt{adityagolatkar@ucla.edu aachille@amazon.com yuxiangw@cs.ucsb.edu}\\
\texttt{aaroth@cis.upenn.edu
mkearns@cis.upenn.edu soattos@amazon.com}
}

\maketitle

\begin{abstract}
We introduce \textit{AdaMix}, an adaptive differentially private algorithm for training deep neural network classifiers using both private and public image data. While pre-training language models on large public datasets has enabled strong differential privacy (DP) guarantees with minor loss of accuracy, a similar practice yields punishing trade-offs in vision tasks. A few-shot or even zero-shot learning baseline that ignores private data can outperform fine-tuning on a large private dataset. AdaMix incorporates few-shot training, or cross-modal zero-shot learning, on public data prior to private fine-tuning, to improve the trade-off. AdaMix reduces the error increase from the non-private upper bound from the 167-311\% of the baseline, on average across 6 datasets, to 68-92\% depending on the desired privacy level selected by the user. AdaMix tackles the trade-off arising in visual classification, whereby the most privacy sensitive data, corresponding to isolated points in representation space, are also critical for high classification accuracy. In addition, AdaMix comes with strong theoretical privacy guarantees and convergence analysis.
\end{abstract}

\section{Introduction}
\label{sec:intro}

When training a deep neural network for visual classification, it is critical to protect privacy by limiting the amount of information that could be gleaned about any individual training sample. Differential Privacy (DP) \cite{dwork2006calibrating} is a theoretical framework to provide strong guarantees on the maximum amount of information that any attacker can extract about an individual training sample. In particular, DP allows users to select the desired trade-off between privacy and accuracy, mediated by a \textit{privacy parameter} $\epsilon$, which typically depends on the application scenario.

Training large machine learning models while guaranteeing strong privacy for each sample is challenging. In practice, however, one often has a pool of data available for which there are no privacy concerns. This could be a synthetic datasets or datasets designed to be available for public use. Such \textit{public} data are distinct from the \textit{private} ones, for which we seek strong privacy guarantees.\footnote{Note that here public data is not the same as data from public sources, as the latter may still require privacy guarantees.}
In particular, using large amounts of generic public data for pre-training has recently enabled the creation of language models that achieve DP on the target task while remaining close to state-of-the-art performance  \cite{li2021large,yu2021differentially}. The same strategy in vision \cite{tramer2020differentially}, however, still yields punishing error increases (\cref{table:mixdp}) ranging from 311\% to 167\% on average across datasets, depending on the desired level of privacy.

\textit{Can we, then, make use of public data that is better suited for vision, so we can retain close-to-paragon performance while ensuring privacy?}

One sure way of preserving privacy is not to use private data altogether, and recent literature has given us multiple ways to do so. For example, using zero-shot learning \cite{larochelle2008zero,lampert2009learning,lei2015predicting,socher2013zero,frome2013devise} one can leverage public data from a different modality (e.g., text) to train a visual model without ever seeing the private data. More generally, one can source or synthesize a few samples of labeled public data from the task distribution,
and train using few-shot learning \cite{wang2019few}, still without using private data. Surprisingly, simple few-shot techniques outperform private training on much larger dataset (Fig. \ref{fig:boost}, left).

Of course, there may be subtle domain shifts between the public and private data, so ignoring the latter is not a desirable strategy to maintain privacy. The question, then, is how to use even small amounts of public data, in addition to the private data, to break the trade-off between privacy and accuracy. To do so, we change the setting from most work on DP to include \textit{labeled} public data sourced with the \textit{same} labels as the target task. We call this setting \textit{mixed differential privacy}, or MixDP.

To address MixDP, we propose to use the public data not just for pre-training the backbone, but for few-shot or zero-shot learning of a classifier on the target tasks, prior to private fine-tuning (Sec.~\ref{sec:adamix}).
In Tab.~\ref{table:mixdp}, we show that indeed in the MixDP setting it is possible to achieve significant gains compared to training with only the private or only the public data, even using a small amount of the latter. To do so, we have to modify existing DP training algorithms to the mixed setting, which leads us to the development of AdaMix, a method for MixDP that uses public data to tune and adapt all major steps of private training, in particular model initialization, clipping of the gradients, and projection onto a lower-dimensional sub space. Some of these ideas to use auxiliary public data were applied in different contexts in isolation, but have suboptimal privacy/performance trade-offs for visual tasks. %

In visual classification tasks, the long-tails of the data
play an important role in achieving high classification performance \cite{feldman2020does,feldman2020neural}. However, DP is a worst-case framework which is highly-influenced by outliers or long-tails. MixDP eases the problem, since it allows to collect public data to ensure that each sub-population is sufficiently covered. This reduced the cost of privacy for the long tails, as we show using a per-instance DP (pDP) analysis of \cite{wang2019per}.%

Unlike most related work, we use NoisyGD rather than DP-SGD as we find that it is faster and more private for our use case. We give a tight analysis of its $(\epsilon, \delta)$-DP properties using the Analytical Gaussian mechanism \cite{balle2018improving,dong2021gaussian} (Sec.~\ref{sec:dp}). In addition, we present a convergence proof for our algorithm together with new stronger bound for the strongly convex case. This also allows us to describe the utility of the algorithm relative to its non-private counterpart.

To summarize, our contributions are: \textbf{(1)} we introduce AdaMix, a state-of-the-art adaptive method for mixed privacy learning; \textbf{(2)} we show significant improvement w.r.t.\ baselines on a novel benchmark for MixDP vision tasks; \textbf{(3)} we show that the zero-shot text information can improve  the performance of private models (Sec.~\ref{sec:experiments}); \textbf{(4)} we analyze the goodness of our method both theoretically (new theorems) and empirically (pDP analysis\textbf{}). %

\section{Related Work}
\label{sec:related}

Most work on Differential Privacy \cite{chaudhuri2008privacy,song2013stochastic,wang2019differentially,wang2019differentially,yu2021differentially,bassily2014private,shokri2015privacy,iyengar2019towards,jayaraman2018distributed,lee2018concentrated} uses public data either for: generic pre-training unrelated to the task \cite{tramer2020differentially}, or to tune parameters \cite{kairouz2020fast,zhou2020bypassing,yu2021not}, or as additional \textit{unlabeled} data \cite{papernot2016semi,papernot2018scalable}. Instead, we use a \textit{small amount of labeled public data related to the task} to improve the accuracy of a private model under a given privacy parameter $\epsilon$.

\paragraph{Unlabeled public data.}
PATE \cite{papernot2016semi,papernot2018scalable,uniyal2021dp} trains teacher models on different partitions of the data to predict labels, and privately vote on how to label public (unlabelled) data. Then, a student model is trained in semi-supervised fashion on public data, now partially labelled using the private data. This strategy does not directly allow to train a model on a mix of labeled public and private data. \cite{zhu2020private} noted that the number of partitions required for meaningful privacy guarantees may be prohibitive for vision datasets. As an alternative, they suggest labeling with a differentially private KNN. In MixDP, where we assume that the public data is already labeled, these algorithms would perform at best similarly to our \textit{Only-Public} baseline. Instead, what we are interested in effectively using labeled private data to improve performance relative to already labeled public data.

\paragraph{DP for Language models.} \cite{li2021large,yu2021differentially} show that a large language model pre-trained on generic public data can be fine-tuned on task-specific private data with only modest loss in accuracy.
In contrast, our focus is on using small amounts of public data for fine-tuning.
We note that in language models strong transfer of information from the pre-training is facilitated by the fact that the output space (a sequence of tokens) is the same for all tasks. Conversely, vision models need to relearn different output spaces based on the task (i.e., the last layer is always trained from scratch) reducing the transfer of information. We show that we can partly recover this advantage by either using an initialization based on public data \textit{of the same task}, or a \textit{multi-modal} vision-language model. We use  CLIP \cite{radford2021learning} for zero-shot learning.

\paragraph{Adaptive clipping.} Similarly to us, \cite{wang2020differentially} studies DP training with access to a small amount of public data, and use adaptive gradient clipping \cite{zhang2021understanding,andrew2019differentially,bagdasaryan2019differential,van2018three} based on the public gradients. Contrary to us, they first train a private model separately and then a mixed public-private model using the private one as a  regularizer with compelling results on tabular data. We found this approach ineffective for computer vision, where usually models are pre-trained on public data. In Sec.~\ref{sec:experiments} we show that AdaMix works better on vision tasks, and is further improved with multi-modal data.%

\paragraph{Adaptive projection.}
\cite{kairouz2020fast,zhou2020bypassing,yu2021not} suggest improving DP-SGD by projecting the gradients into a lower-dimensional space estimated from public data. 
However, these methods do not leverage the performance improvements derived from training on public data, which are needed to make vision models viable.
Using public data both to train and to estimate the subspace introduces more challenges, since the gradients of the public data and private data will then behave differently. Moreover, while they perform an SVD on the individual public sample gradients, we do the SVD of the total gradient matrix. Unlike the SVD of the matrix of individual sample gradients, this is cheaper to compute at each step.

\paragraph{Analysis.} \cite{alon2019limits} studies the theoretical limits of mixed privacy learning, but does not provide practical algorithms. We provide both a method for MixDP and a proof of convergence.
Our analysis leverages the results of %
\cite{lacoste2012simpler} which shows that a particular non-uniform averaging scheme of SGD converges at an $O(1/t)$ rate, rather than the classical $O(\log t/t)$ rate. In the context of differentially private learning, \cite{bassily2014private} showed that NoisySGD achieves the optimal rates for (convex) empirical risk minimization (ERM) \cite{chaudhuri2011differentially}. \cite{song21evading} provided the first analysis of NoisyGD in the general convex case.
To the best of our knowledge, we are the first to formally establish that the use of an arbitrarily small public data gives provably smaller sample complexity in learning compared to the standard private learning setting.

\section{Method}

\paragraph{Notation.} Let $z\in \cZ$ be a datum, with $z=(x, y)$ being the feature/label pair; $\cZ^*=\cup_{n=0}^\infty \cZ^n$ is the universe of datasets with arbitrary size.
For learning, we have access to a private dataset $D_{\pri} := \{z_1,...,z_{N_{\pri}}\}$ and a public one $D_{\pub}:=\{\tilde{z}_1,...,\tilde{z}_{N_{\pub}}\}$. The loss function $\ell: \cZ\times\cW\rightarrow \R$ takes data $z\in\cZ$ and ``weights'' $w\in\cW \subset \R^d$ to yield the total loss on the private and public dataset $\cL_\pri(w):=\sum_{i=1}^{N_\pri}\ell_i(w)$ and $\cL_\pub(w):=\sum_{j=1}^{N_\pub}\tilde{\ell}_j(w)$, with $\ell_i$ and $\tilde{\ell}_j$ short-hands for $\ell(w,z_i)$ and $\ell(w,\tilde{z}_j)$ respectively. Notice that $\cL_\pri$ and $\cL_\pub$ are sums, not averages, of the loss functions.
Their gradients are indicated by $g_i^\pri(w) = \nabla_w \ell_i(w)$, $g_j^\pub(w) = \nabla_w \tilde{\ell}_j(w)$, $G^\pri(w) = \nabla_w \mathcal{L}_\pri(w)$, $G^\pub(w) = \nabla_w \mathcal{L}_\pub(w)$.
In addition, we say $N = N_\pri + N_\pub$ and $\cL(w) = \cL_\pri(w) + \cL_\pub(w)$.

The average \textit{empirical risk} is written as $\hat{\cR}(w):=\cL(w)/N$. When the data are drawn from some distribution $\cD$, the risk, i.e., generalization error, $\cR(w):=\E_{z\sim \cD}[\ell(z,w)].$ Our goal is to design a differentially private algorithm that returns $\hat{w}$ which minimizes $\cL(\hat w)$. Its performance is measured by the \textit{excess empirical risk}, i.e., $\hat{\cR}(\hat{w}) - \min_{w\in\cW}\hat{\cR}(w)$; and by the \textit{excess risk} $\cR(\hat{w}) -\min_{w\in\cW}\cR(w)$. In the experiments, the latter is measured by the error on a test set.

\subsection{Differential privacy}
\label{sec:dp}
Two datasets $D$, $D' \in \mathcal{Z}^*$ are said to be \textit{neighboring}, written as $D \simeq D'$, if we can obtain $D'$ from $D$ by \textit{adding or removing} a single data point.
\begin{definition}[Differential privacy{\cite{dwork2006calibrating,dwork2014algorithmic}}]\label{def:dp} For $\epsilon > 0$ and $\delta \ge 0$, a randomized algorithm $\mathcal{M}$ is $(\epsilon, \delta)$-differentially private if for any neighboring datasets $D \simeq D^\prime$ and any measurable $S \subseteq \mathrm{range}(\cM)$,
\begin{align*}
\mathsf{Pr}[\mathcal{M}(D) \in S] \le e^\epsilon \cdot \mathsf{Pr}[\mathcal{M}(D^\prime) \in S] + \delta.
\end{align*}
\end{definition}
In our context, $\cM$ is the randomized learning algorithm, and $\cY$ is the space of model weights. This condition bounds the maximum amount of information that an attacker, having access to the weights of the model, can extract about one of the training samples. The parameter $\epsilon$ is referred to as the \textit{privacy parameter}:
lower values of $\epsilon$ carry better privacy guarantees, but also make it harder to train an accurate model, as it reduces the amount of information about the training samples that can be represented via the parameters of the model. The goal of a good DP training algorithm $A$ is to train an accurate model while satisfying the $(\epsilon,\delta)$-privacy constraint selected by the user where, ideally, $\epsilon \approx 1$ and $\delta = o(1/n)$.

\paragraph{NoisyGD.} NoisyGD is defined in Algorithm~\ref{alg:noisyGD}. We now show that NoisyGD is differentially private and that, in the convex case, it converges to the optimal solution.

\begin{algorithm}[t]
\caption{(Projected) Noisy Gradient Descent}
\label{alg:noisyGD}
\KwData{Constraint set $\cW$, initialization $w_1$, noise level $\sigma$, clipping threshold $\tau$, number of iteration $T$, weight decay parameters $\lambda$ and $w_\text{ref}$, learning rates $\eta_t$.}
\KwResult{$w_1,...,w_{T+1}$}

\For{$t = 1,\ldots,T$}{
$G_t = \sum_{i=1}^{N_\pri} \operatorname{clip}_\tau (g_i(w_t))$\;
$n_t \sim N(0, \sigma^2 \tau^2 I)$\;
$w_{t+1} \gets w_t - \eta_t (G_t + n_t + \lambda (w_t - w_\text{ref}))$\;
$w_{t+1} \gets \Pi_{\cW}(w_{t+1})$\;
}
\end{algorithm}

\begin{proposition}[NoisyGD is $(\epsilon,\delta)$-DP, informal]
\label{prop:noisygd-is-private-informal}
Algorithm~\ref{alg:noisyGD}
with parameter $T,\sigma^2$ such that $\rho:= \frac{T^2}{2\sigma^2}$ satisfies the $(\epsilon,\delta(\epsilon))$-DP, where $\delta(\epsilon) := \Phi(\frac{\mu}{2}-\frac{\epsilon}{\mu}) - e^\epsilon \Phi(-\frac{\mu}{2}- \frac{\epsilon}{\mu})$, $\mu = \sqrt{2\rho}$ and $\Phi$ is the CDF of the normal distribution.
\end{proposition}
For any prescribed privacy parameter $(\epsilon,\delta)$, it suffices to choose this $\rho$ parameter, thus from here onward we will use $\rho$ as the privacy parameter %
of interest (note that, asymptotically, $\rho \asymp \min\{\epsilon, \epsilon^2/\log(1/\delta)\}$).

\paragraph{Mixed private learning.} We define the problem of differentially private learning when we have access to an additional public dataset \emph{mixed private learning}. We say the algorithm $\cM$ (now taking two arguments $D_{\pri}$ and $D_{\pub}$) satisfies $(\epsilon,\delta)$-DP if $\cM(\cdot,D_{\pub})$ satisfies $(\epsilon,\delta)$-DP (Def.~\ref{def:dp}) for every fixed $D_{\pub}$.
\cite{alon2019limits} shows that a small public dataset from the same distribution allows one to privately PAC-learn any classes with bounded VC-dimension, but with an inefficient algorithm. We focus on designing practical DP algorithm
that can effectively use both public and private data under the setting of convex empirical risk minimization \cite{chaudhuri2011differentially} and convex fine-tuning of deep models \cite{achille2021lqf}.

\subsection{AdaMix}
\label{sec:adamix}

\paragraph{Model.} Following \cite{tramer2020differentially}, we use a linear logistic regression model
trained on top of the last layer features of an ImageNet pre-trained network. This reduces the expressivity of the model compared to training a full network. However, the significant reduction in number of parameters makes it a better choice DP algorithms and recent deep networks allow for competitive performance on most tasks even when training only the last layer. We normalize the features before feeding them to the logistic model, which both improves the accuracy and bound the gradients at each step.

\paragraph{Initialization and regularization.}
The choice of initialization is normally of minor concern for logistic regression: since the problem is convex we will converge to a global optimum regardless of the initialization. L2 regularization (weight decay), when employed, is usually centering at $0$ by default. However, centering the regularizer is critical in differential privacy, as we prove in the following result:

\begin{theorem}[Convergence of NoisyGD on strongly convex problems]
\label{thm:excess_empirical_risk_bound}
Assume $\cW$ is a convex set satisfying $\sup_{w\in\cW}\|w\|\leq B$. Let $w_1,...,w_{T+1}$ be the parameter vectors returned by running Algorithm~\ref{alg:noisyGD}
on the following regularized loss
$J(w) = \cL(w) + \frac{\lambda}{2}\|w- w_{\text{ref}}\|^2$
with $w_1 = w_{\text{ref}}\in \cW$,
learning rate $\eta_t = \frac{2}{\lambda(t+1)}$ and $\lambda = \sqrt{\frac{\rho}{2\|w^*-w_{\text{ref}}\|d L^2}}$.
Then, the ensemble solution:
\[
\textstyle \bar{w} = \frac{2}{T(T+1)} \sum_{t=1}^T t\, w_t
\] obeys the bound
\[
\textstyle
\E[\hat{\cR}(\bar{w})]- \hat{\cR}(w^*) \leq \frac{4(N \tau + \lambda B)^2 }{\lambda TN} + \frac{\sqrt{d} \tau\|w^*-w_{\text{ref}}\|}{\sqrt{2\rho}N}
\]
for any $w^*\in\cW$, where $\rho$ is the privacy parameter of the algorithm.
\end{theorem}

Note that as $T \to +\infty$ the first term vanishes and the second term matches the information-theoretic limit for privately solving general convex ERM problems when $w_{\text{ref}} = 0$ \cite{bassily2014private}.
However, this can improve if we choose $w_{\text{ref}}$ according to the public data such that $w_{\text{ref}}$ is closer to $w^*$ than $0$. This motivated us to propose choosing $w_{\text{ref}}$ by training on the public data.
Then, we proceed to fine-tune the model on the concatenation of private and public data, using Noisy-GD while preserving $(\epsilon, \delta)$-DP.

\paragraph{Provable benefits of public data access.}
Next we provide a theoretical guarantee for the above approach under standard statistical learning assumptions.

\begin{theorem}[Generalization with access to public data, informal]
\label{thm:provable_benefit_of_public_data_informal}
Assume that private and public data are drawn i.i.d.\ from the same distribution $\cD$ and that $R(w) = \E_{(x,y)\sim \cD}[ \ell(w, (x,y))]$ is $c$-strongly convex in $\cW$. Let $w_{\mathrm{ref}}$ be obtained by training for one epoch on the public dataset and let $\bar{w}$ be as in Thm.~\ref{thm:excess_empirical_risk_bound}, then at the limit when $T$ is sufficiently large, the \emph{excess risk} obeys
\begin{align*}
&\E[\cR(\bar{w})] - \cR(w^*) \\
\leq& \frac{4\sqrt{d} L^2}{c\sqrt{N_\pub}N\sqrt{2\rho} }+ \text{Gen}(\bar{w},N) + \text{Gen}(w^*,N),\end{align*}
where $w^* = \argmin_{w\in\cW} \cR(w)$ and
$
\text{Gen}(w,N):= \left|\E[\cR(w) - \hat{\cR}(w)]\right|
$ is the expected generalization gap of (a potentially data-dependent) $w$.
\end{theorem}

To the best of our knowledge, Theorem~\ref{thm:provable_benefit_of_public_data_informal} is the first demonstrated claim that access to even a very small public dataset can improve private learning in the convex ERM setting.
We note that the generalization bound decomposes in a first term, which describes the generalization cost incurred by using DP, and a second part which is common to non-private learners. Interestingly, in the Appendix we show that in MixDP the cost of privacy asymptotically vanishes even if the percentage of public data becomes neglegible (i.e., when $N_\pub\rightarrow \infty$, $N_{\pub}/N\rightarrow 0$).

Besides providing a good initialization, access to a small public dataset also enables ways to further improve Noisy-GD in practice, as we will now see.

\paragraph{Adaptive threshold.} Since the norm of the gradients changes significantly during training, a fixed clipping threshold $\tau$ is suboptimal throughout training. \cite{wang2020differentially} suggests to tune this threshold adaptively using public data. However, the norm of the average gradient of the public sample, may not yield good clipping thresholds for the \textit{per-sample} gradients. Instead, we find the following adaptive threshold to be more robust:
\[
\tau = \operatorname{quantile}_{90}(\{\|g_i^\pub(w)\|\}_{i=i}^{N_\pub}),
\]
that is, we take the 90-th quantile of the norm of the public gradients, where $g_i^\pub(w) = \nabla_w \tilde{\ell}_i(w)$ is the gradient of the $i$-th public sample.
Then the clipped gradient we use is
\[
\tilde{g}_i^\pri(w) = \frac{g_i^\pri(w)}{\|g_i^\pri(w)\|} \min(\|g_i^\pri(w)\|, \tau).
\]
This quantile rule is also used by \cite{andrew2019differentially}, but we differ in the use of public data to determine $\tau$, rather than allocating part of the privacy-budget for releasing the quantile.
It was established in \cite{chen2020understanding} that gradient clipping has the effect of Huberization of the loss function in the GLM cases. The convergence to the optimal solution of a modified objective function is guaranteed if we fix the clipping parameter. Our approach on the adaptive clipping has the effect of selecting the hyperparameter adaptively in a homotopy style algorithm, which shows further empirical benefits.

\paragraph{Adaptive subspace projection.} From Thm.~\ref{thm:provable_benefit_of_public_data_informal}, the utility of a DP algorithm decreases with the number of parameters $d$ of the model. In view of this, we do not train the whole network but rather use it as a feature extractor and only train a linear model on top. However, for modern vision architectures, even the number of parameters of the final layer is in the order of tens of thousands. Therefore, we use an adaptive mechanism to project private gradients onto a low-dimensional subspace before adding noise. More precisely,
\[
G^\pub(w) = \nabla_w L_{D_\pub}(w) = \sum_{i=1}^{N_\pub} \nabla_w \ell (\tilde{z}_i, w)
\]
be the total gradient of the public dataset $D_\pub$. Since we are using a linear logistic model, $G^\pub(w)$ is a $D \times C$ matrix where $D$ is the number of features and $C$ is the number of classes.
Consider the SVD decomposition:
\[
G^\pub(w) = U S V.
\]
We now project each (clipped) private gradient matrix $\tilde{g}_i^\pri$ on the top $P$ components, and obtain a matrix $\hat{g}_i^\pri = U^t \tilde{g}_i^\pri$ of size $P \times C$.
Other methods projecting the gradient using public data \cite{zhou2020bypassing,yu2021not,kairouz2020fast} do not compute the SVD of the total gradient matrix, as noted in Sec.~\ref{sec:related}. Rather, use the vectorized \textit{per-sample} gradients of the public examples to compute the principal components to project. Since, in vision, the size of matrix of all per-sample gradients has a much larger size $N \times FC$ (rather than our $F \times C$), several approximations would be required to carry out this strategy.

\paragraph{Final algorithm.} The final gradient update step at time $t$ is:
\[
w_{t+1} \gets w_t - \eta_t ( G^\pub + U \hat{G}^\pri + \lambda \cdot w_t)
\]
where
\[
\textstyle \hat{G}^\pri = \sum_{i=1}^{N_\pri} U^T \tilde{g}^\pri_i + n_t.
\]
and $n_t \sim \cN(0, \sigma^2 \tau^2 I_C)$. We then train for $T$ steps. The pseudo-code of algorithm is given in Algorithm~\ref{alg:mix-s}.
\begin{proposition}
Algorithm~\ref{alg:mix-s}
with parameter $T,\sigma^2$ such that $\rho:= \frac{T^2}{2\sigma^2}$ satisfies the $(\epsilon,\delta(\epsilon))$-DP of a Gaussian mechanism in Theorem 8 \cite{balle2018improving} with parameter $\mu = \sqrt{2\rho}$.
\end{proposition}
Note that the privacy guarantee is the same as long as the ratio $\sigma/\sqrt{T}$ remains constant, which gives a degree of freedom in choosing $\sigma$ (hence $T$). %
Higher values of $\sigma$ requires more training steps $T$, but, by Thm.~\ref{thm:excess_empirical_risk_bound}, it also ensures better convergence (which we validate this empirically in Sec.~\ref{sec:experiments}). Hence, the user should generally pick the largest $\sigma$ allowed by their computational budget.

\begin{algorithm}[t]
\caption{AdaMix training algorithm.
}\label{alg:mix-s}
\KwData{Public dataset $D_\pub$, private dataset $D_\pri$, privacy parameter $(\epsilon, \delta)$, noise variance $\sigma$, learning rate $\eta$, $L_2$ reg. coefficient $\lambda$}
\KwResult{$w = A(S)$}
$T_\text{max} \gets \operatorname{Calibrate}(\epsilon, \delta, \sigma)$\;
$w \gets \operatorname{Pretrain}(S_u)$\;
\For{$t = 1,\ldots,T_\text{max}$}{
\Comment{Adaptive clipping}
$\tau_t \gets \operatorname{quantile}_{90}\big(\{\|g_i^\pub(w)\|\}_{i=1}^{N_\pub}\big) $\;
\Comment{Adaptive projection}
$U, S, V \gets \operatorname{SVD}(G_\text{public}(w))$\;
\Comment{Clip, project and add noise to private gradients}
$n_t \sim N(0, \sigma^2 \tau_t^2 I)$\;
$\tilde{g}^\pri_i(w) \gets \frac{g_i^\pri(w)}{\|g_i^\pri(w)\|} \min(\tau_t, \|g_i^\pri(w)\|)$\;
$\tilde{G}^\pri(w) = \sum_{i=1}^{N_\pri} U^T \tilde{g}^\pri_i(w) + n_t$\;
\Comment{Final weights update}
$w \gets w - \eta (G^\pub(w) + U \tilde{G}^\pri(w) + \lambda \cdot w)$\;
}
\end{algorithm}

\begin{table*}
\centering
\small
\setlength{\tabcolsep}{2pt}
\begin{tabular}{|lc|c|c|ccc|ccc|}
\hline
& Public Shots & Non-Private & Only-Public & Fully-Private & PPGD\cite{wang2020differentially} & AdaMix & Fully-Private & PPGD\cite{wang2020differentially} & AdaMix\\
& & (paragon) & (baseline) &\multicolumn{3}{c|}{$\epsilon=1$} &\multicolumn{3}{c|}{$\epsilon=3$} \\
\hline
Ox.\ Flowers \cite{Nilsback06}\!\! & 2 & 12.34 & 40.16 & 95.12 & 40.99 & \textbf{39.48} & 81.42 & 39.22 & \textbf{36.11} \\
CUB-200 \cite{WelinderEtal2010} & 2 & 31.97 & 64.01 & 96.79 & 64.20 & \textbf{63.58} & 81.61 & 61.59 & \textbf{59.15}\\
Stanf.\ Dogs \cite{KhoslaYaoJayadevaprakashFeiFei_FGVC2011}\!\! & 2 & 10.08 & 16.58 & 33.87 & 16.42 & \textbf{15.38} & 15.93 & 15.61 & \textbf{12.72}\\
MIT-67 \cite{quattoni2009recognizing} & 5 & 25.17 & 43.58 & 69.55 & 42.02 & \textbf{40.46} & 42.99 & 39.60 & \textbf{33.03}\\
Oxford Pets \cite{parkhi12a} & 5 & 7.17 & 12.83 & 26.02 & 12.85 & \textbf{11.54} & 12.37 & 11.15 & \textbf{9.53} \\
Caltech-256 \cite{griffin2019caltech} & 5 & 14.64 & 24.88 & 60.72 & 24.61 & \textbf{23.74} & 27.15 & 23.71 & \textbf{20.85}\\
\hline
\end{tabular}
\caption{\textbf{Mixed privacy, full privacy and AdaMix.}
We report the result (test errors) of different methods on a diverse set of vision tasks (see text for description). Surprisingly \textbf{Only-Public}, which throws away private data and only trains a model on the small amount of public data outperforms \textbf{Fully-Private}, which trains on all data as if it was private. \textbf{AdaMix} is instead able to effectively use both the private and public data at the same time, and achieves a significant improvement even at relatively small value of $\epsilon$ (e.g., 10\% improvement on MIT-67 w.r.t.\ Only-Public). Instead, the closest method to us, \textbf{PPGD} does not significantly improve over the Only-Public baseline.
}
\label{table:mixdp}
\end{table*}

\subsection{Textual side information}
\label{sec:multi-modal}

CLIP \cite{radford2021learning} is an embedding trained so that the representation of an image is close to the representation of its textual description.
This enables zero-shot learning, i.e., initializing an high-performance image classification model using only the name of the labels or a textual description of the classes, which is generally public.
AdaMix can be easily adapted to work in this multi-modal MixDP setting, where the public data comes from a different modality. First, we initialize a linear logistic model using the names of the classes: Let $c_i$ be the text associated to the $i$-th class, the weight matrix is \cite{radford2021learning}:
\[
W = \beta [\enc(c_1), \ldots, \enc(c_C)],
\]
that is, each row is the normalized encoding of a label scaled by a parameter $\beta$ (we use $\beta=100$).
Starting from this public initialization, we train with AdaMix on image data.

\section{Results}
\label{sec:experiments}
\vspace{-5pt}
\paragraph{Setting.} Unless otherwise specified, we train a linear logistic model on top of the last layer features of a ResNet-50 \cite{he2016deep} pre-trained on ImageNet \cite{deng2009imagenet}. For NoisyGD, we initalize the logistic weights to zero, fix $\sigma=20$ and $\delta=10^{-5}$. We train with $\lambda=1e-2$, weight decay centered at zero, learning rate $\eta \in \{0.0005,0.001,0.0025,0.005\}$, and select the best learning rate for each value of $\epsilon$. For AdaMix, we initialize the weights by training on the public data until convergence. We use the Gaussian mechanism implementation of \texttt{autodp}\footnote{\url{https://github.com/yuxiangw/autodp}} to calibrate the parameters of the training algorithm, in particular $T_{\max}$, in order to achieve the desired $(\epsilon, \delta)$-DP requirement.

\paragraph{Datasets.} We test our algorithm on the following vision classification datasets commonly used in transfer learning (see Tab.~\ref{table:mixdp}).
We split each training dataset into a ``public'' dataset consisting of $N$-samples per class (see Tab.~\ref{table:mixdp}), and consider the remaining training data as private. In all cases, the public data is less than $10\%$ of the private data.

\paragraph{Baselines considered.}
We compare AdaMix with the following baselines:
\textbf{(Fully-Private)} Ignore the MixDP setting and simply train on all data as if it was private, starting form a zero initialization and using NoisyGD;
\textbf{(Only-Public)} Train a standard non-private model using only the few public samples and discarding the private ones;
\textbf{(PPGD)} A version of PPSGD \cite{wang2020differentially}, but trained with NoisyGD rather than DP-SGD for easier comparison. Unlike AdaMix, PPGD has several hyper-parameters: we manually searched for the best combination to ensure the best comparison;
\textbf{(NGD)} An ablated version of AdaMix which uses fixed hyper-parameters rather than the adaptive components.

\subsection{Experiments}
\vspace{-5pt}
\paragraph{Mixed Privacy vs Full Privacy.}
In \Cref{table:mixdp} we see that, perhaps surprisingly, Only-Public is always significantly better than Fully-Private even if it actually throws away most of the available data. This suggests that, in vision, collecting a small amount of public data may be preferred to collecting a much larger amount of private data (Fig.~\ref{fig:boost}). However, we see that AdaMix is able to effectively use both the private and public data at the same time, and achieves a significant improvement even at relatively small value of $\epsilon$ (e.g., 10\% improvement on MIT-67 w.r.t.\ Only-Public at $\epsilon=3$). Conversely, the method closest to us, PPGD, does not significantly improve over the Only-Public baseline even after our changes to adapt it better to vision tasks.

\begin{figure*}
\centering
\includegraphics[width=.99\linewidth]{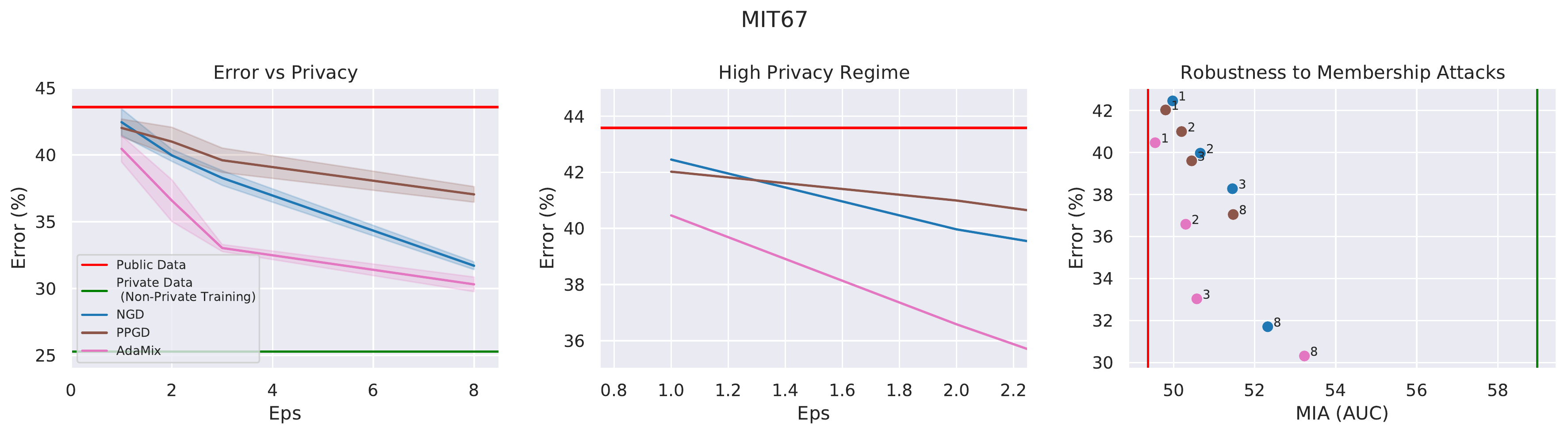}
\caption{\textbf{Test error vs privacy for several methods.} \textbf{(Left)}
On MIT-67, we show the test error obtained by different methods at different privacy parameters $\epsilon$. The top horizontal red line shows the perfectly private Only-Public baseline (not using the private data at all). The bottom green line shows the non-private paragon (training on all data as if they where public). By changing the privacy parameter $\epsilon$, we can interpolate between the two extrema.
We see that AdaMix
performs uniformly better than other methods, and can use private and public data together to obtain significant boost at all $\epsilon$.
\textbf{(Center)} Same as the before, but we zoom in to show the behavior of the different solutions in the low-$\epsilon$ regime.
\textbf{(Right)} For the same algorithms, we plot the robustness to adversarial attacks (measured by AUC) vs test error obtained using different values of $\epsilon$ (annotated near each point). AdaMix obtains the best trade-off between accuracy and robustness.
}
\label{fig:comparison}
\end{figure*}

\paragraph{Ablation study and finer comparison.} In Fig.~\ref{fig:comparison} (left and center) we plot the average accuracy obtained by various methods and ablations of AdaMix for different privacy parameters $\epsilon$. Both AdaMix and the ablated NGD are able to improve over the Only-Public baseline. Interestingly, this holds true even when considering relatively low privacy parameters $\epsilon \approx 1$ which is usually considered too restrictive to train in a fully private setting (as seen in Tab.~\ref{table:mixdp}). On the other hand, thanks to the MixDP training we can obtain high-accuracy models even at restrictive privacy settings. We note that AdaMix performs uniformly better than other methods at all privacy regimes. From the ablation of AdaMix, we see that initializing the model with the public data has the largest impact, followed by adaptive clipping and adaptive projection.

\paragraph{Membership attacks.} Differential privacy is a very strong requirement, and is it often assumed that using higher value of $\epsilon$ may still provide valuable defense to attacks \cite{shokri2017membership,carlini2019secret,carlini2020extracting,rahman2018membership,bernau2019assessing,zhu2020deep,sablayrolles2019white,lecuyer2019certified,ma2019data} even if the theoretical guarantees are weak. Conversely, \cite{jagielski2020auditing} suggests using attacks to give a lower-bound on the privacy parameter of the algorithm. In Fig.~\ref{fig:comparison} (right) we plot the Area Under Curve of a membership attack executed against models for different values of $\epsilon$.
We use a thresholding based membership attack as proposed in \cite{sablayrolles2019white}. The loss for samples in the private training set are labelled as 0 and test samples are labelled as 1. Then we simply compute the AUC for this. We find that the sorting of the algorithms does not change: algorithms with better theoretical bounds are also more robust to attacks.

\paragraph{Performance Boost and size of public data.}
We hypothesize that, when enough public data is present, the contribution of using private data may be less noticeable. To test this, we define the Performance Boost of a Mixed Privacy method as $\text{PB} = (\text{err}_\text{public} - \text{err}_\text{paragon}) / (\text{err}_\text{DP} - \text{err}_\text{paragon})$, where $\text{err}_\text{paragon}$ is the error obtained by training on non-private model on both public and private data, $\text{err}_\text{public}$ is the error obtained by training only on the public data, and $\text{err}_\text{DP}$ is the error obtained by using training a mixed privacy model. This measures how well we can learn from the private data in the MixDP setting compared to the non-private upper-bound. In Fig.~\ref{fig:boost} we plot the performance boost for different sizes of the public set. We observe that the performance boost derived of using mixed privacy is very high when the public data is scarce, disappears when the amount of public data is comparable or larger than the private data.

\paragraph{Multi-modal initialization.}
In Fig.~\ref{fig:multimodal} we compare CLIP models initialized with (1) a random initialization, (2) a zero-shot initialization based on the public label names (Sec.~\ref{sec:multi-modal}), and (3) using few public image samples. While the latter performs the best at a given privacy level $\epsilon$, the zero-shot initialization still improves significantly over the random initialization, thus showing that even just small amount of textual information (which is often available) can be used to significantly boost accuracy.
Multi-modal models may also learn a better lower-dimensional representation of the data that is more amenable to DP. In Fig.~\ref{fig:multimodal} we compare AdaMix using an ImageNet pretrained ResNet-50 and using CLIP. This is not a direct comparison, but we see a significant performance boost by using CLIP at a given privacy level $\epsilon$, suggesting that the move to multi-modal model may be a logical next step for DP.

\begin{figure}[b]
\centering
\includegraphics[width=.98\linewidth]{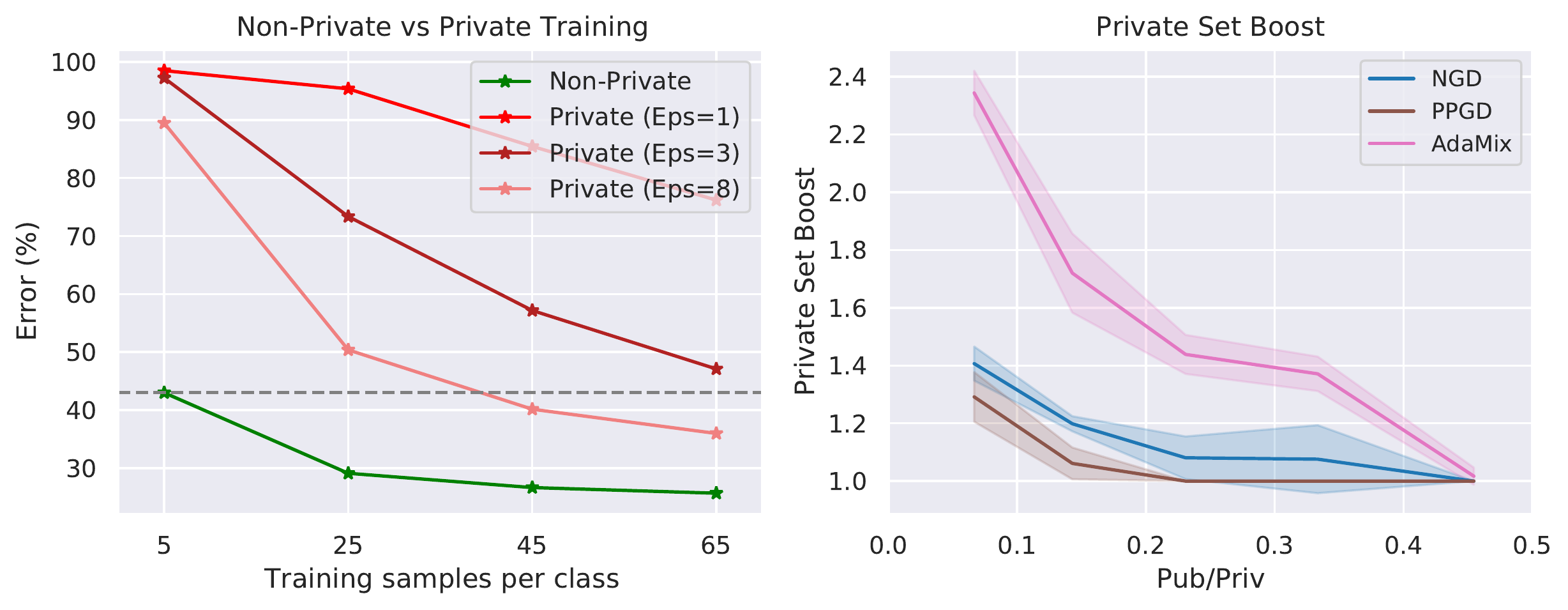}
\caption{\textbf{(Left) Trade-off between public and private examples.} We show the accuracy reached on MIT-67 using $N$ samples per class (x-axis) when they are public (green) or private (in red for different values of $\epsilon$). Fully private training requires as much as 10 times more samples. This leaves users with a trade-off between collecting more private data or, if possible, a smaller amount of public data.
\textbf{(Right) Performance boost using a private set for different methods.} We plot the Performance Boost (PB) for different DP methods as the ratio of public to private samples changes (see text for definition). We see that AdaMix achieves the best performance boost. As expected, when the public set is relatively large, the boost from adding private data decreases.%
}
\label{fig:boost}
\end{figure}

\begin{figure*}[t]
\centering
\includegraphics[width=0.9\linewidth]{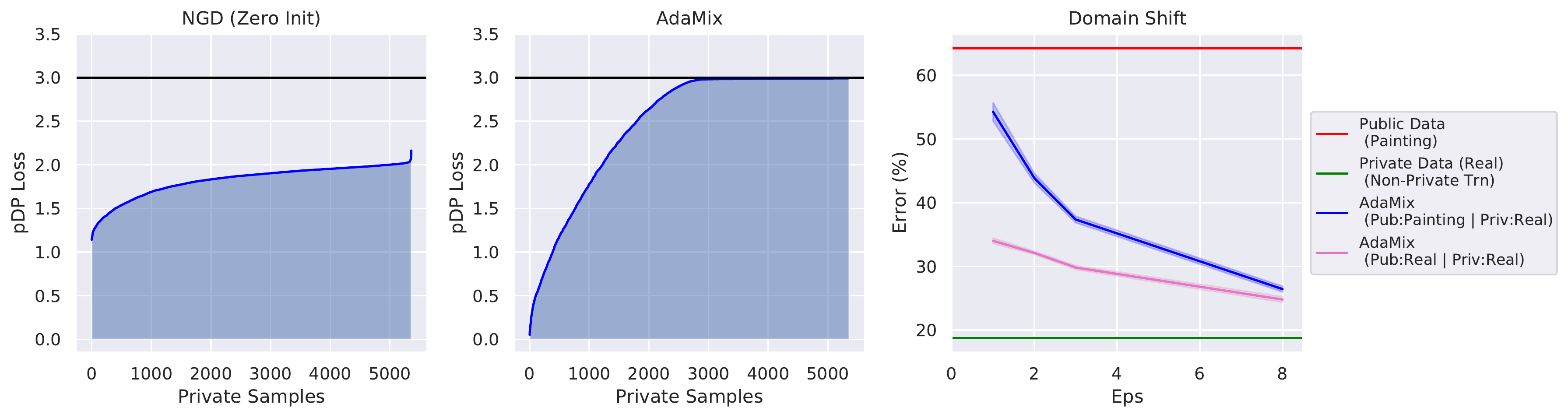}
\caption{\textbf{(Left and center) Effect of public data on per-instance DP.} For $\epsilon=3$, we plot the sorted pDP scores for each sample in the dataset with Fully-Private training \textbf{(left)} and with AdaMix \textbf{(center)}. AdaMix makes the the target $\epsilon$ (horizontal line) closer to pDP value used by most samples (blue curve).
\textbf{(Right) MixDP and domain-shifts.}  We plot the performance of AdaMix when (purple curve) the public data come from the same domain as the test (real images), and (blue curve) from different domains (public images are paintings). We see that AdaMix use the private data from the target domain (real images) to adapt the solution to the target task. This significantly improves over using only the public data from the wrong domain (red line) and for large $\epsilon$ approaches the non-private limit (green line).
}
\label{fig:individual}
\end{figure*}

\paragraph{Effect of differential privacy on individual samples.}
DP is a worst-case guarantee which may penalize the accuracy of a model on unusual data, or data on the long-tails of the distribution, which however may play an important role for generalization in vision \cite{feldman2020does,feldman2020neural}. MixDP may ease the problem, since it allows to collect public data to ensure that each sub-population is sufficiently covered, lessening the accuracy cost of making that data private. We show this effect using the per-instance DP (pDP) analysis of \cite{wang2019per}, which measures the privacy loss $\epsilon'$ incurred by each individual $z$ during the private training. This allows us to provide a finer analysis of different methods by comparing the histogram of pDP loss $\epsilon$ for each individual in the dataset $D$ on top of the worst case DP bounds, which we do in Fig.~\ref{fig:individual} (see Appendix for details and theoretical analysis).
We observe that, when using AdaMix, the pDP losses of most samples are the same, and is close to the maximum privacy budget $\epsilon=3$,  indicating that we are using the prescribed privacy budget more effectively and more uniformly across samples, thereby improving utility.

\paragraph{Domain shift.} Often, the private and public data may come from slightly different domains. To test the robustness of AdaMix to domain shifts, we use the DomainNet dataset, which aims to classify objects in different domains. In particular, we train a model to classify objects in painting, using however public data coming from real pictures (while the private data is in the painting domain). In Fig.~\ref{fig:individual} (right), we show that, even if the public comes from a different domain, AdaMix is still able to use the public data to significantly improve performance of private training.

\paragraph{Effect of $\sigma$.} Outside of the learning rate $\eta$, the main other sensitive hyper-parameter of AdaMix is the noise variance $\sigma$. From Prop.~\ref{thm:excess_empirical_risk_bound} we expect that the best accuracy will be obtained using a large $\sigma$ and training for a respectively larger amount of epochs $T$. In the Appendix, we plot the test error as a function of $\sigma$ and we see that indeed, a larger $\sigma$ give better results at the expense of longer training time.

\vspace{-8pt}
\section{Discussion}
\vspace{-1pt}
We study Mixed Differential Privacy learning in computer vision, and show that in this setting differential privacy can be used to provide high accuracy models while respecting a given privacy parameter. To do this, we introduce AdaMix, an adaptive differential privacy mechanism that uses the public data to initialize and compute, at each step, an optimal subspace and clipping threshold. We also show that multi-modal vision and text can significantly improve accuracy under a privacy constrain.

\paragraph{Limitations.} In our analysis we train a linear model on pre-trained features, rather than fine-tuning the whole network. This is common in DP, since fine-tuning multiple layers may not significantly improve the accuracy while the additional parameters negatively affects the privacy bounds. This limitation is partially offset by new models that are trained to have generic and transferable last-layer features without the need of additional fine-tuning. We discuss several theoretical bounds that guide us in understanding the behavior of the algorithm. However, these bounds refer to a slighlty simplified version of AdaMix (see Appendix for details) and use hyper-parameters that are asymptotically optimal, but may not be practical.

\paragraph{Conclusions.} Differentially Private models in computer vision often cannot reach the same level of performance as non-private model for realistic privacy parameters. Using AdaMix in the MixDP learning setting, we have shown that, assuming the presence of a small amount of public data, accurate networks can be produced that reach better accuracies than fully private training without compromising the privacy of the data. We hope that further research in this field may help the adoption of private models in more areas of computer vision.
\vspace{-1cm}
\begin{figure}[b]
\centering
\includegraphics[width=.98\linewidth]{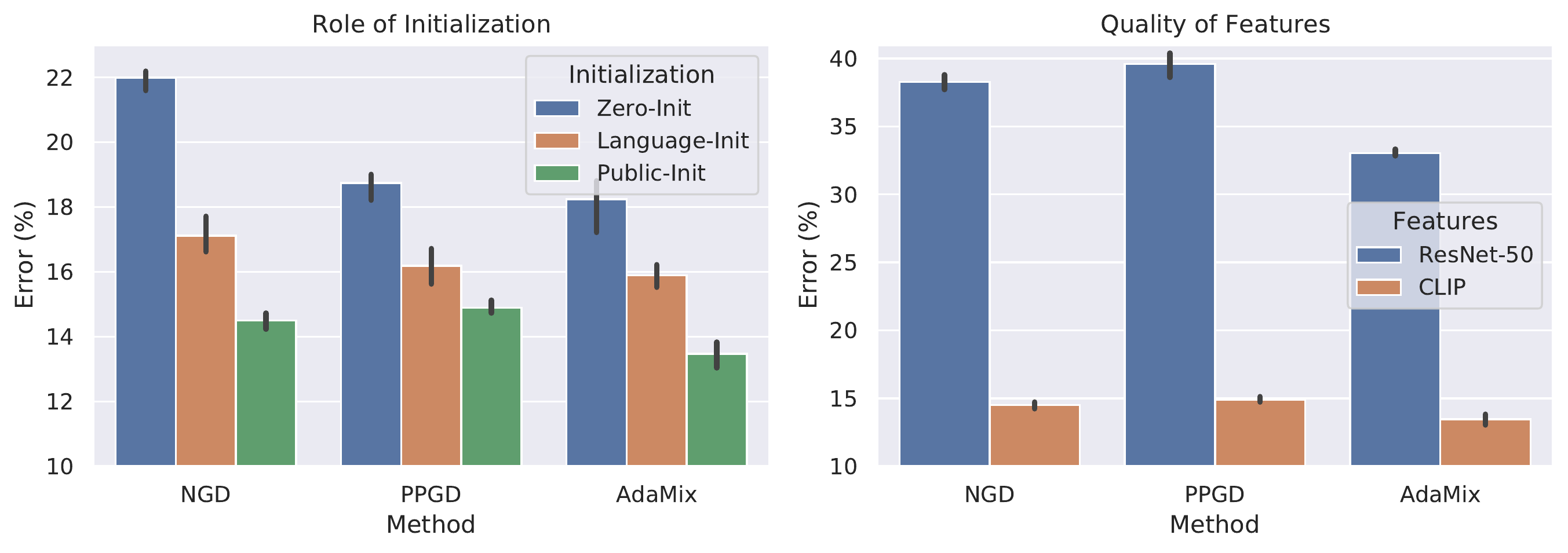}
\caption{\textbf{(Left) Multi-modal initialization is better than zero initialization.} We compare the test accuracy on MIT-67 of a CLIP model trained with $\epsilon=3$ using different initialization strategies. We observe that initializing the logistic weights to zero leads to the worst results under a given privacy parameter. Using the public label names to initialize the weights (language initialization) performs significantly better. However, using the few available public images to initialize the weights still outperforms the other choices.
\textbf{(Right) Multi-modal models.} We compare the performance of ResNet-50 model, and a CLIP model. We see that CLIP performs significantly better under a given privacy parameter. While the ResNet-50 model and the CLIP model are not directly comparable due to different pre-training, this further reinforces that stronger pre-training can indeed benefit the privacy setting and that the feature space learned by multi-modal models may be better suited for privacy than standard vision networks. }
\label{fig:multimodal}
\end{figure}

\clearpage

{\small
\bibliographystyle{ieee_fullname}
\bibliography{egbib,noisygd_refs}
}

\clearpage
\newpage
\appendix

\section{Supplementary Material}

In the supplementary material we perform, (1) Ablation studies analyzing different components of our algorithm (\cref{sec:ablation}), (2) tuning the adaptive clipping threshold can further improve AdaMix (\ref{sec:clip_thresh}), (3) show that using a larger $\sigma$ (training for longer) is better (\ref{sec:longer}), (4) provide additional experimental results (\ref{sec:additional_experiments}), (5) provide details experimental details (\ref{sec:multimodal_appendix}, \ref{sec:exp_detail}), and (6) proofs for all the theoretical results presented in the paper (\cref{sec:dp-basics}, \ref{sec:perinstance-dp}, \ref{sec:proofs}). 

\subsection{Ablation Studies}
\label{sec:ablation}
In \cref{fig:ablation}, we ablate the two components of our algorithm, namely, Subspace Projection and Adaptive Clipping. We consider the Noisy-GD (NGD) as the baseline and analyze the improvement provided by the two components separately and finally in combination (AdaMix). We observe that adaptive clipping is the key component of the AdaMix algorithm. We show that using the two components together performs the best across multiple datasets. In the next section we show that further tuning the adaptive clipping threshold can improve the performance even more on some datasets. In \cref{fig:subspace_proj}, we plot the relative reconstruction error for the private gradients on the public gradient subspace and a random subspace. We observe that the reconstruction error on the public gradient subspace is at least half the random subspace throughout training, which suggests the importance of the public gradients for AdaMix.

\begin{figure}[h]
\centering
\includegraphics[width=\linewidth]{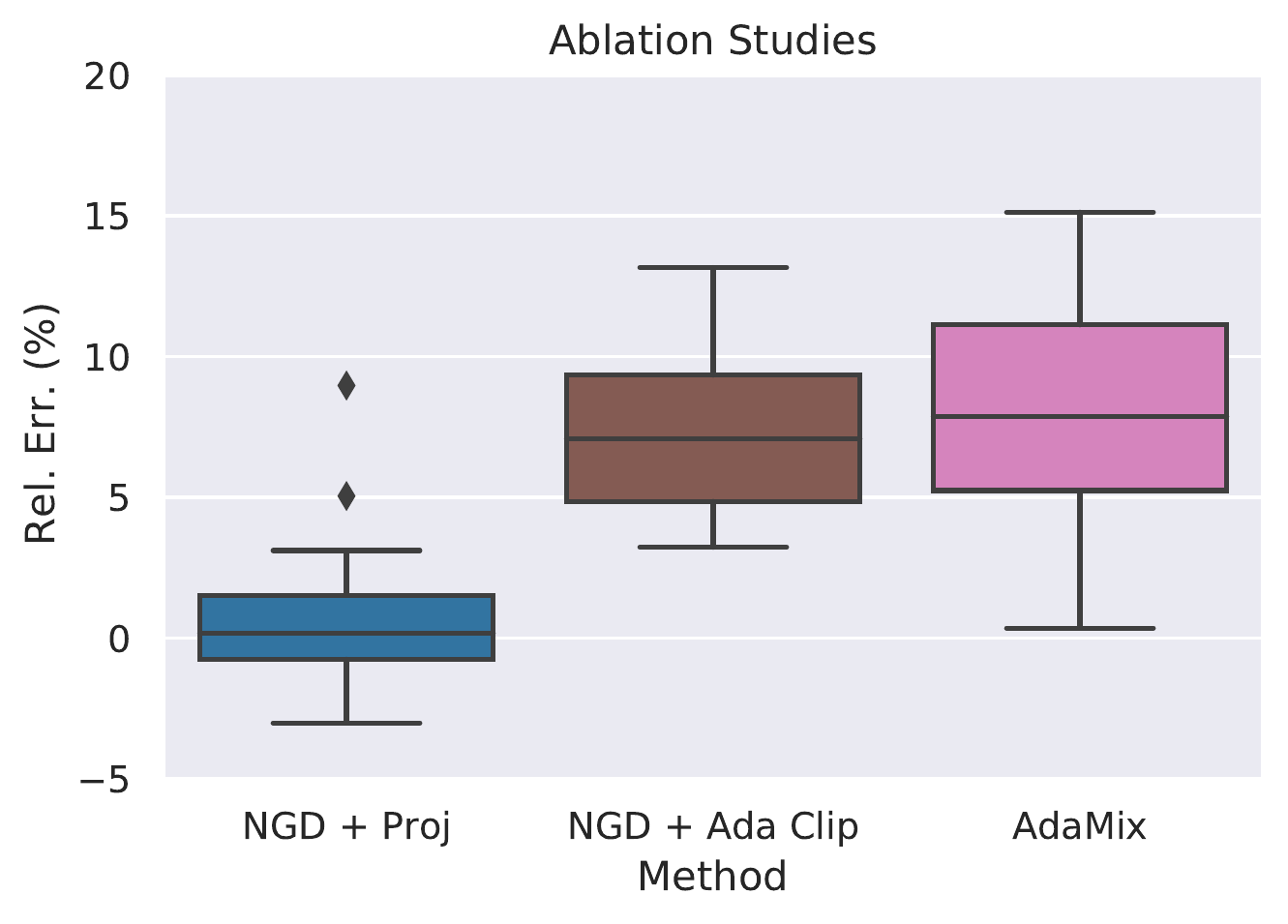}
\caption{\textbf{Ablation Studies} Box plot showing the relative decrease in the test error across multiple datasets when we add Subspace Projection, Adaptive Clipping and Subspace Projection + Adaptive Clipping (AdaMix) to NGD. We observe that adaptive clipping provides more improvement compared to subspace projection, and the combination of the two works the best across 6 datasets.}
\label{fig:ablation}
\end{figure}

\begin{figure}[h]
\centering
\includegraphics[width=\linewidth]{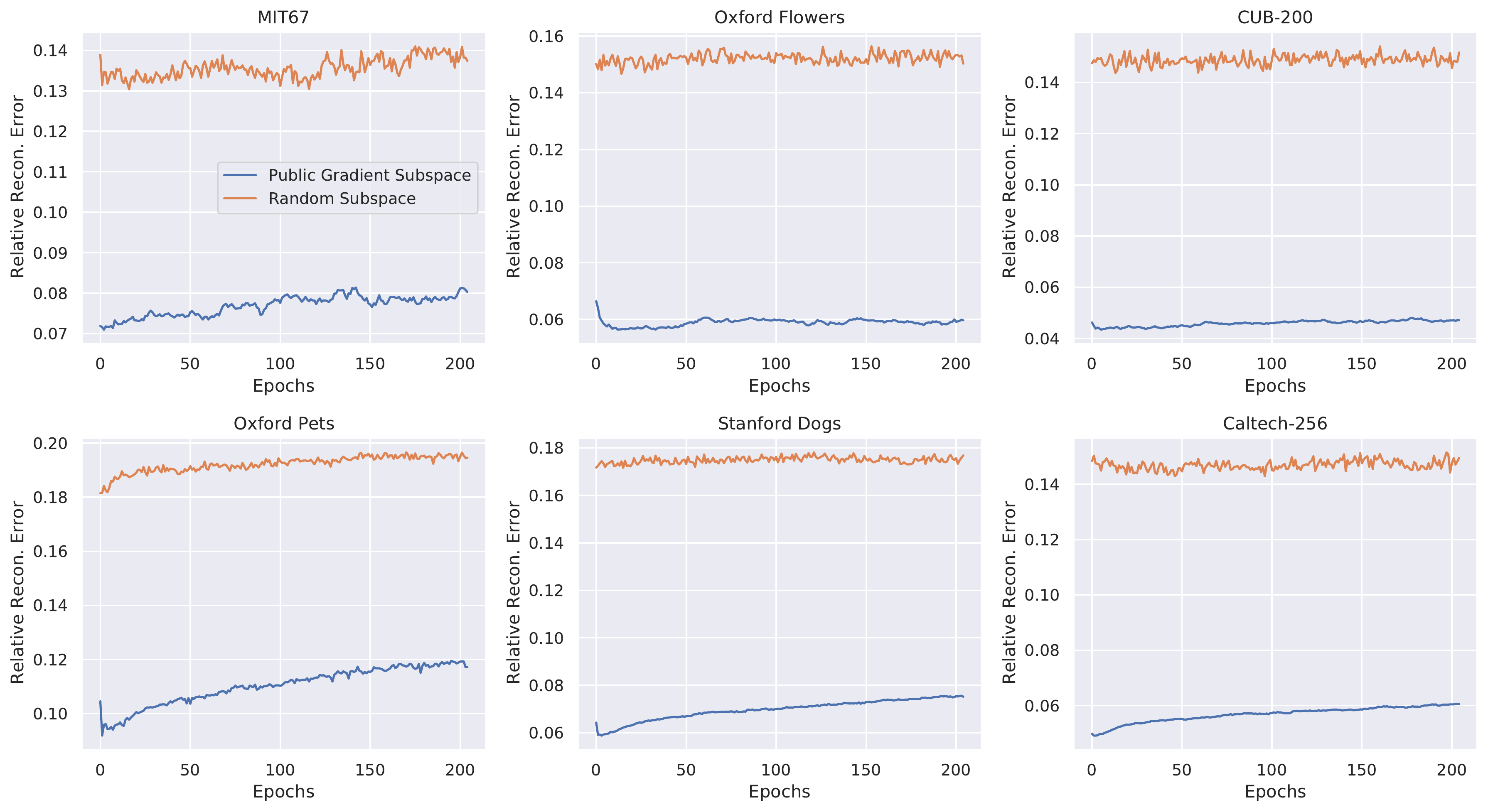}
\caption{\label{fig:subspace_proj}\textbf{Random Subspace Projection} We plot the relative reconstruction error of the private gradients when projecting on the subspace spanned by the public gradients or on a random subspace  during training (using AdaMix and $\epsilon=3$). The reconstruction error when projecting on the random subspace is generally more than twice the error obtained projecting on the subspace of public gradients, highlighting the importance of using the latter.}

\vspace{-1em}

\end{figure}

\vspace{-0.5cm}
\subsection{Effect of AdaMix clip threshold}
\label{sec:clip_thresh}
We plot the effect of adaptive clipping threshold (percentile) on the test error for AdaMix (we use 90 percentile in all of our experiments) in \cref{fig:effect-percentile}. However for CUB-200 and Caltech-256 tuning it to 75 percentile works better indicating that AdaMix has more scope for improvement. Note, however that using 90 percentile for CUB-200 and Caltech-256 still outperforms all the baselines. 

\begin{figure}[h]
\centering
\includegraphics[width=\linewidth]{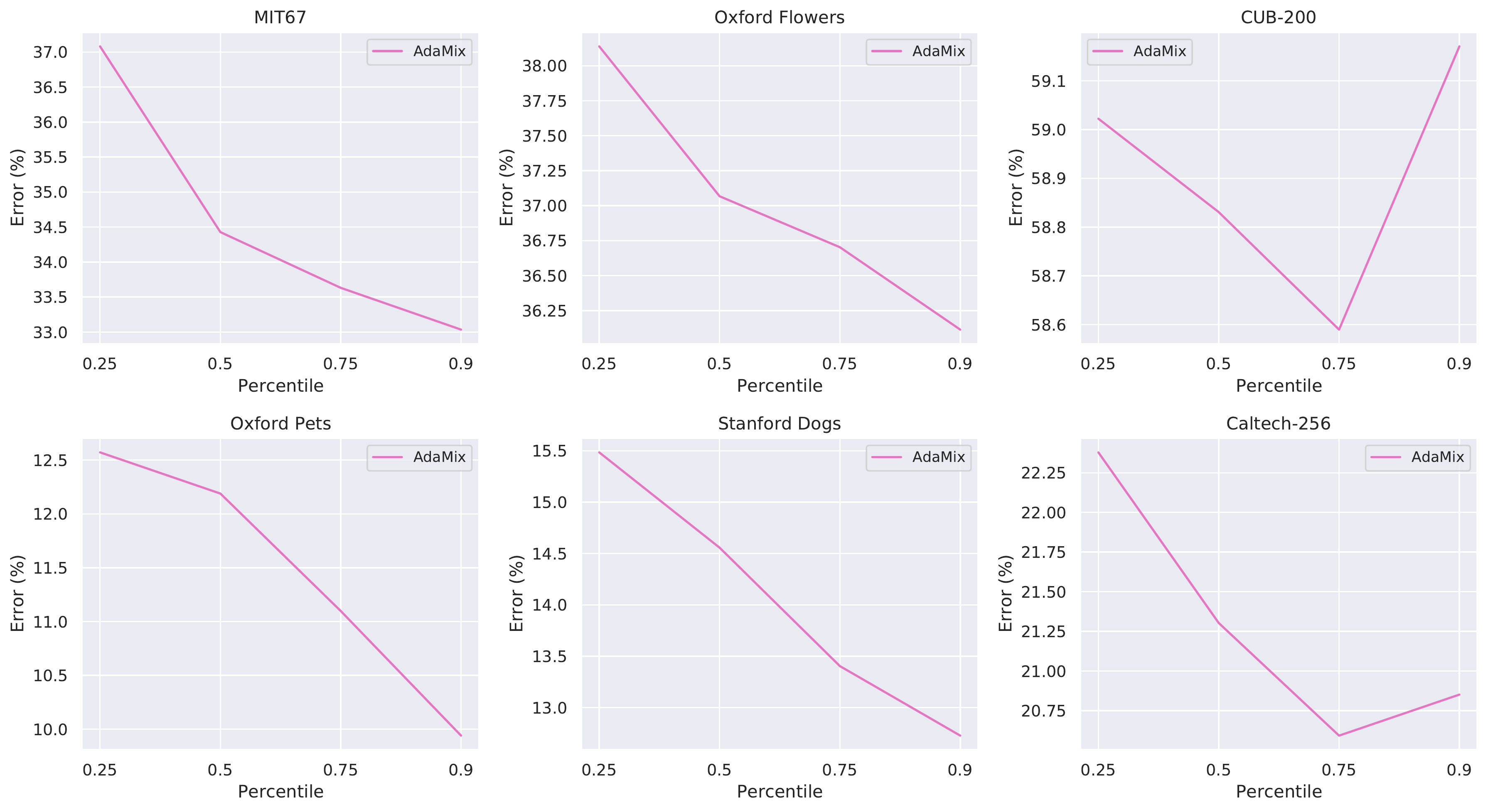}
\caption{\label{fig:effect-percentile}\textbf{Effect of adaptive clipping threshold} We plot the test accuracy of AdaMix using different values of adaptive clipping threshold percentile.}
\end{figure}

\subsection{Longer training with more noise}
\label{sec:longer}

Higher values of $\sigma$ requires more training steps, however, \cref{thm:excess_empirical_risk_bound} shows that it leads to better convergence, which leads to better generalization. In \cref{fig:effect-sigma} we see that indeed, at the same level of privacy,  using a large of $\sigma$ significantly reduces the test error compared to using a smaller $\sigma$.

\begin{figure}[h]
\centering
\includegraphics[width=\linewidth]{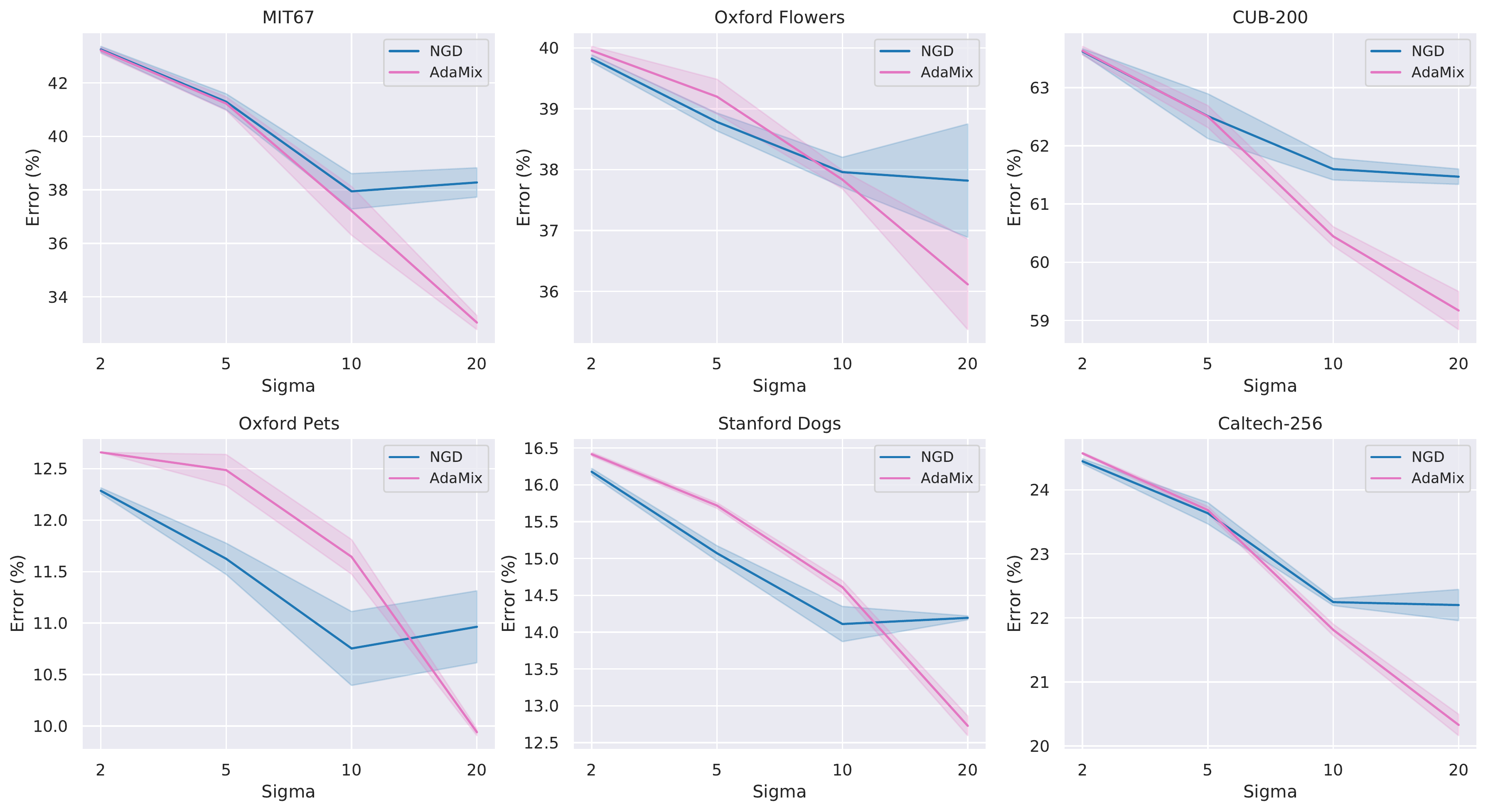}
\caption{\textbf{Larger values of $\sigma$ performs better.} We plot the test accuracy of NGD and AdaMix using different values of $\sigma$ for $\epsilon=3$. A larger $\sigma$ performs better but needs more training steps thus verifying the claim empirically across multiple datasets.}
\label{fig:effect-sigma}
\end{figure}

\subsection{Additional Experiments}
\label{sec:additional_experiments}
In the main paper we presented detailed experimental results on MIT-67, here we present detailed results on all the remaining datasets.

\vspace{-0.5cm}
\subsubsection{Test Error vs Privacy}
\label{sec:err_privacy_appendix}
We show the test error obtained by different methods for different levels of privacy $\epsilon$ and the robustness to membership attack, similar to \cref{fig:comparison} for different datasets in \cref{fig:error_vs_privacy_oxfordflowers,fig:error_vs_privacy_cub200,fig:error_vs_privacy_pets,fig:error_vs_privacy_dogs,fig:error_vs_privacy_caltech}

\begin{figure}[h]
\centering
\includegraphics[width=\linewidth]{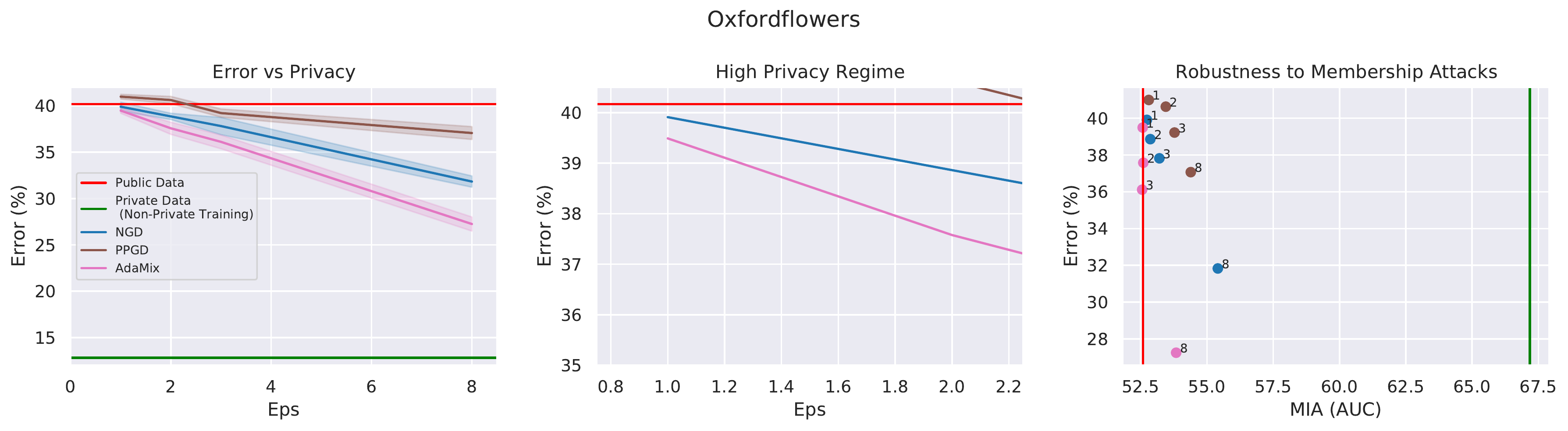}
\caption{Test error vs Privacy and Robustness to Membership Attacks on Oxford-Flowers}
\label{fig:error_vs_privacy_oxfordflowers}
\end{figure}
\vspace{-0.8cm}

\begin{figure}[h]
\centering
\includegraphics[width=\linewidth]{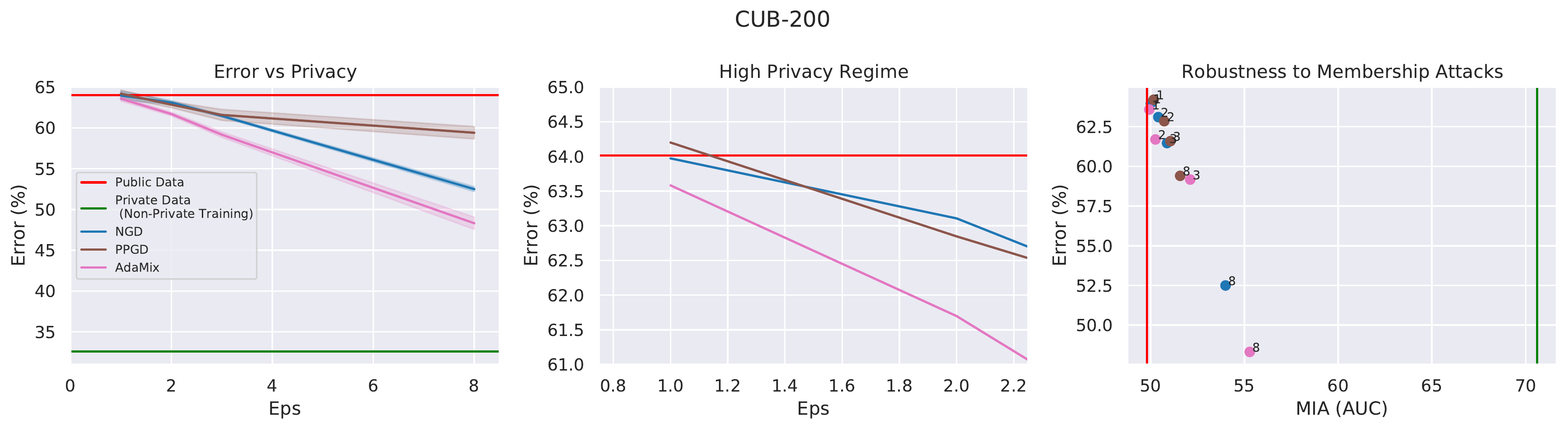}
\caption{Test error vs Privacy and Robustness to Membership Attacks on CUB-200}
\label{fig:error_vs_privacy_cub200}
\end{figure}
\vspace{-0.5cm}

\begin{figure}[h]
\centering
\includegraphics[width=\linewidth]{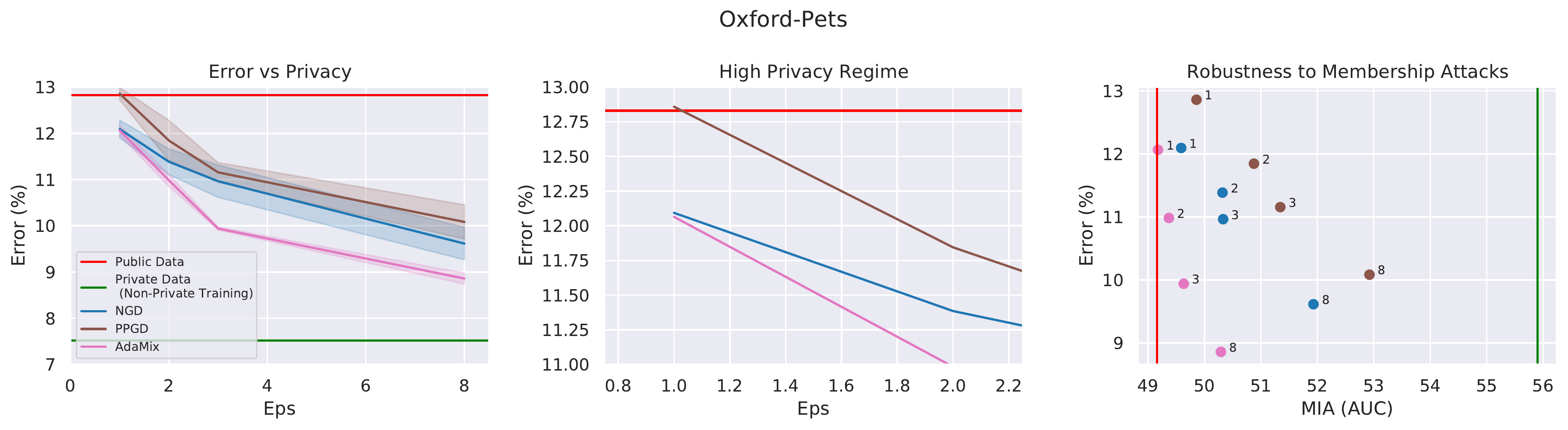}
\caption{Test error vs Privacy and Robustness to Membership Attacks on Oxford-Pets}
\label{fig:error_vs_privacy_pets}
\end{figure}

\begin{figure}[h]
\centering
\includegraphics[width=\linewidth]{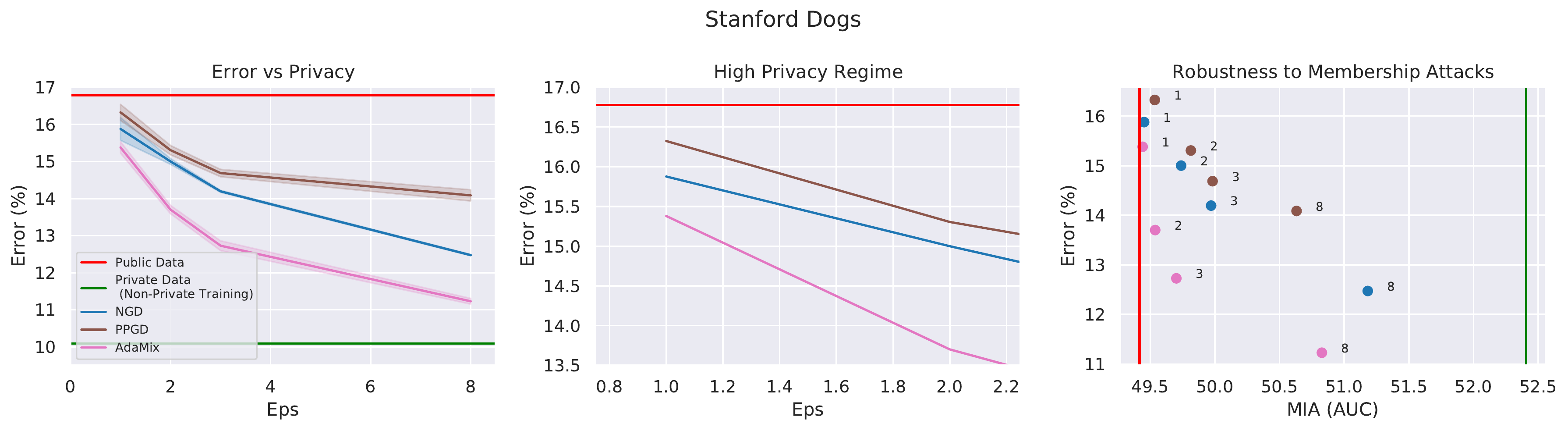}
\caption{Test error vs Privacy and Robustness to Membership Attacks on Stanford Dogs}
\label{fig:error_vs_privacy_dogs}
\end{figure}

\begin{figure}[h]
\centering
\includegraphics[width=\linewidth]{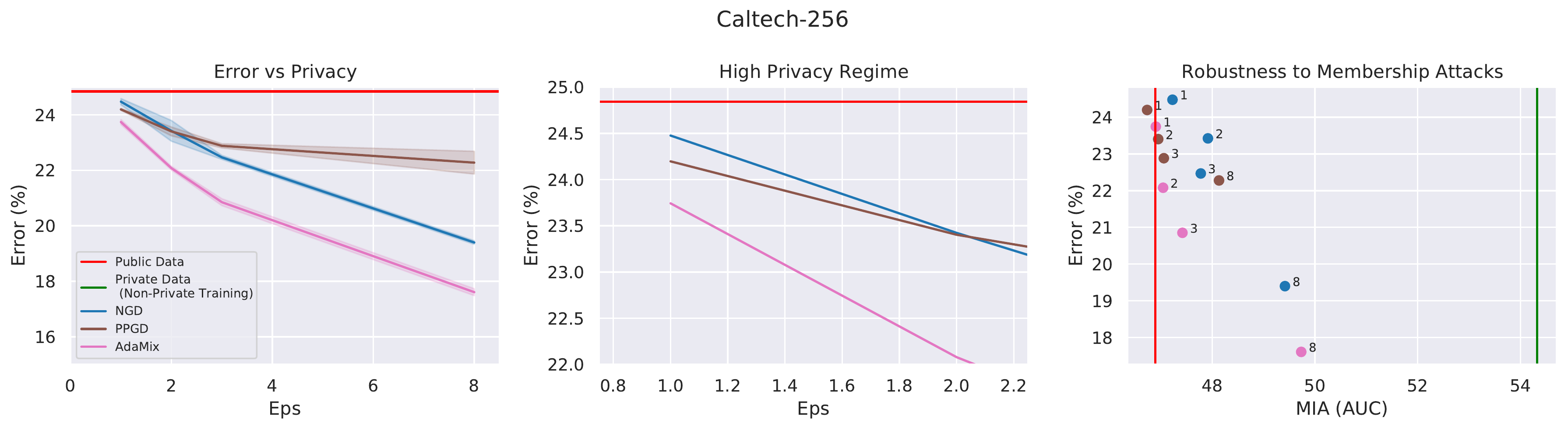}
\caption{Test error vs Privacy and Robustness to Membership Attacks on Caltech-256}
\label{fig:error_vs_privacy_caltech}
\end{figure}

\subsubsection{Per-Instance Privacy}
\label{sec:instance_privacy_appendix}
We show the pDP loss for NGD and AdaMix for different datasets in \cref{fig:pDP_oxfordflowers,fig:pDP_cub200,fig:pDP_pets,fig:pDP_stanforddogs,fig:pDP_caltech256}, similar to \cref{fig:individual} (we use $\epsilon=3$).

\begin{figure}[h]
\centering
\includegraphics[width=\linewidth]{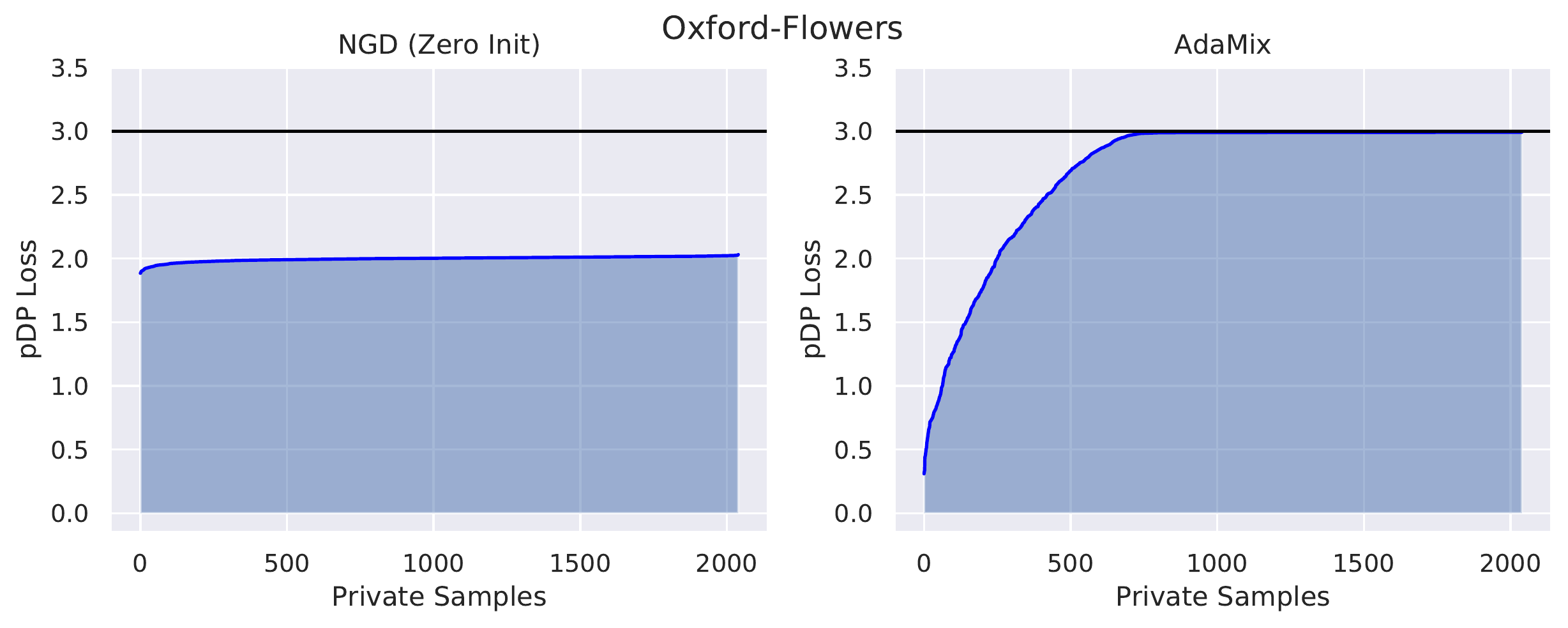}
\caption{Effect of public data on per-instance DP for Oxford Flowers}
\vspace{-1em}
\label{fig:pDP_oxfordflowers}
\end{figure}

\begin{figure}[h]
\centering
\includegraphics[width=\linewidth]{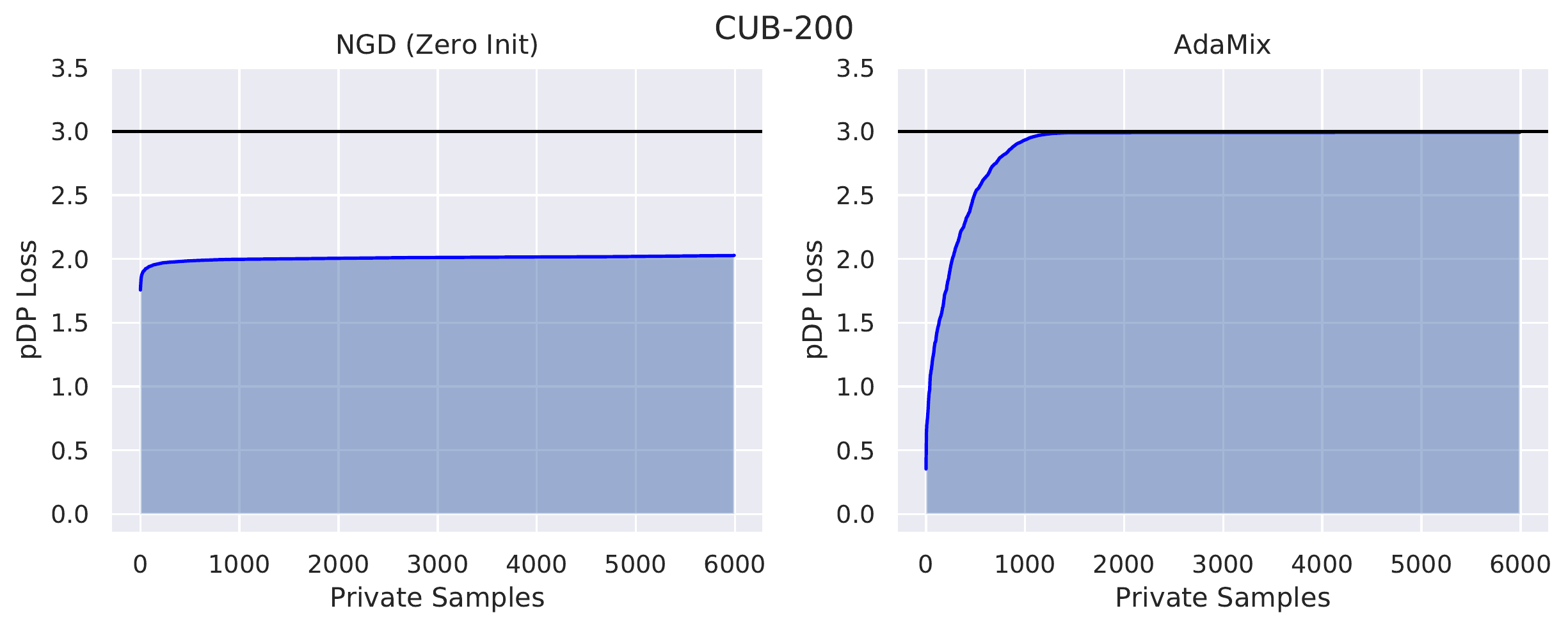}
\caption{Effect of public data on per-instance DP for CUB-200}
\vspace{-1em}
\label{fig:pDP_cub200}
\end{figure}

\begin{figure}[h]
\centering
\includegraphics[width=\linewidth]{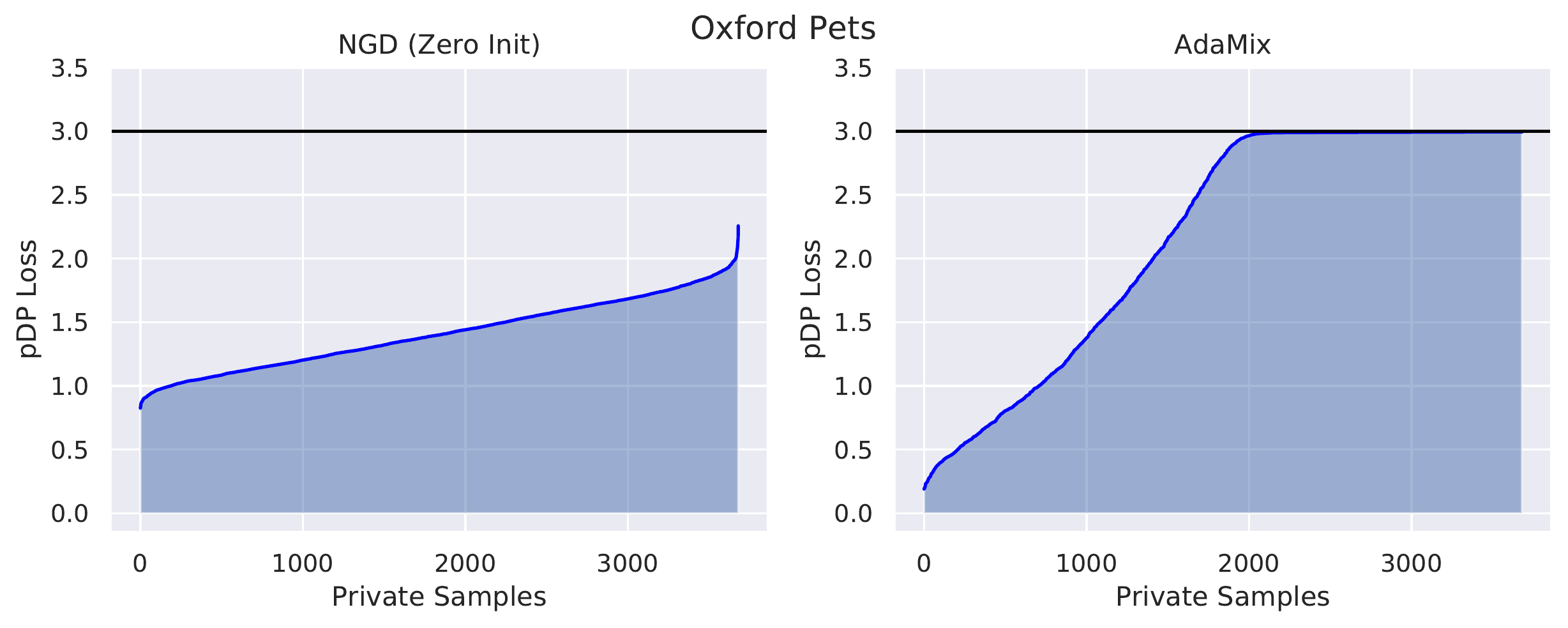}
\caption{Effect of public data on per-instance DP for Oxford Pets}
\vspace{-1em}
\label{fig:pDP_pets}
\end{figure}

\begin{figure}[h]
\centering
\includegraphics[width=\linewidth]{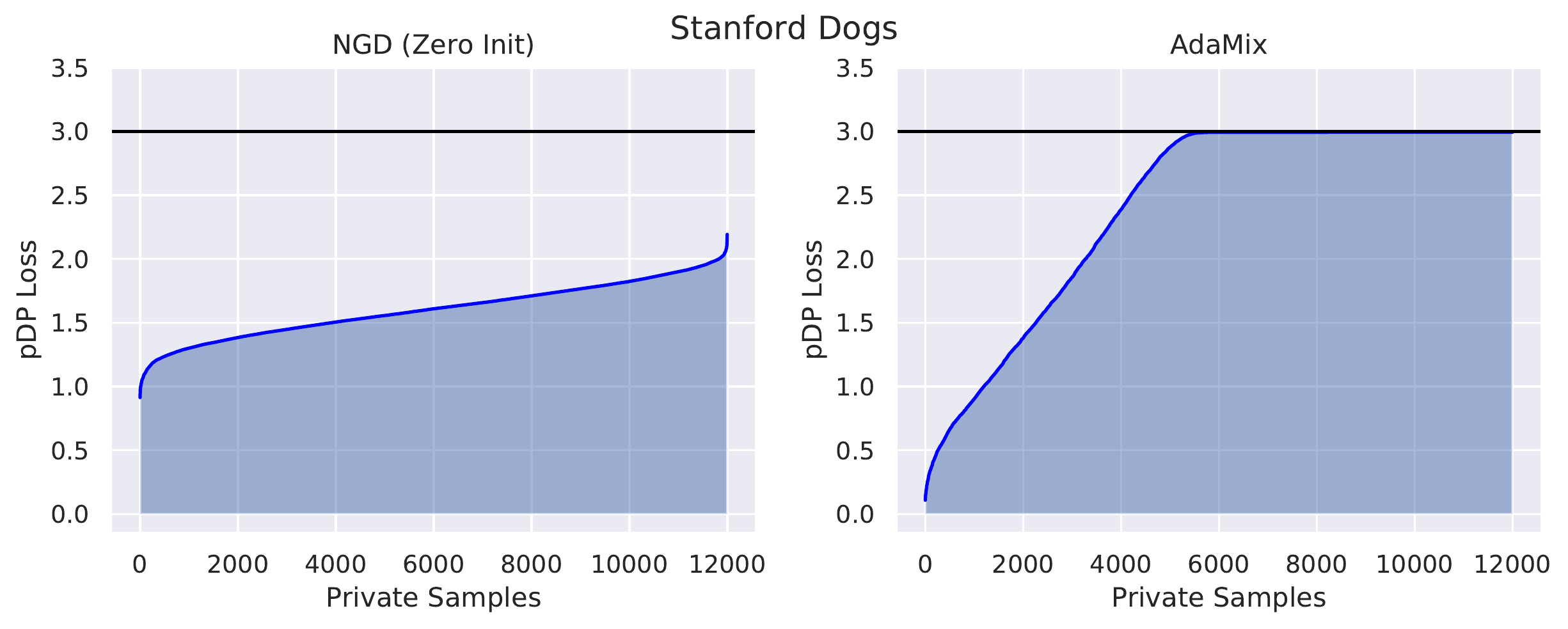}
\caption{Effect of public data on per-instance DP for Stanford-Dogs}
\label{fig:pDP_stanforddogs}
\vspace{-1em}
\end{figure}

\begin{figure}[h]
\centering
\includegraphics[width=\linewidth]{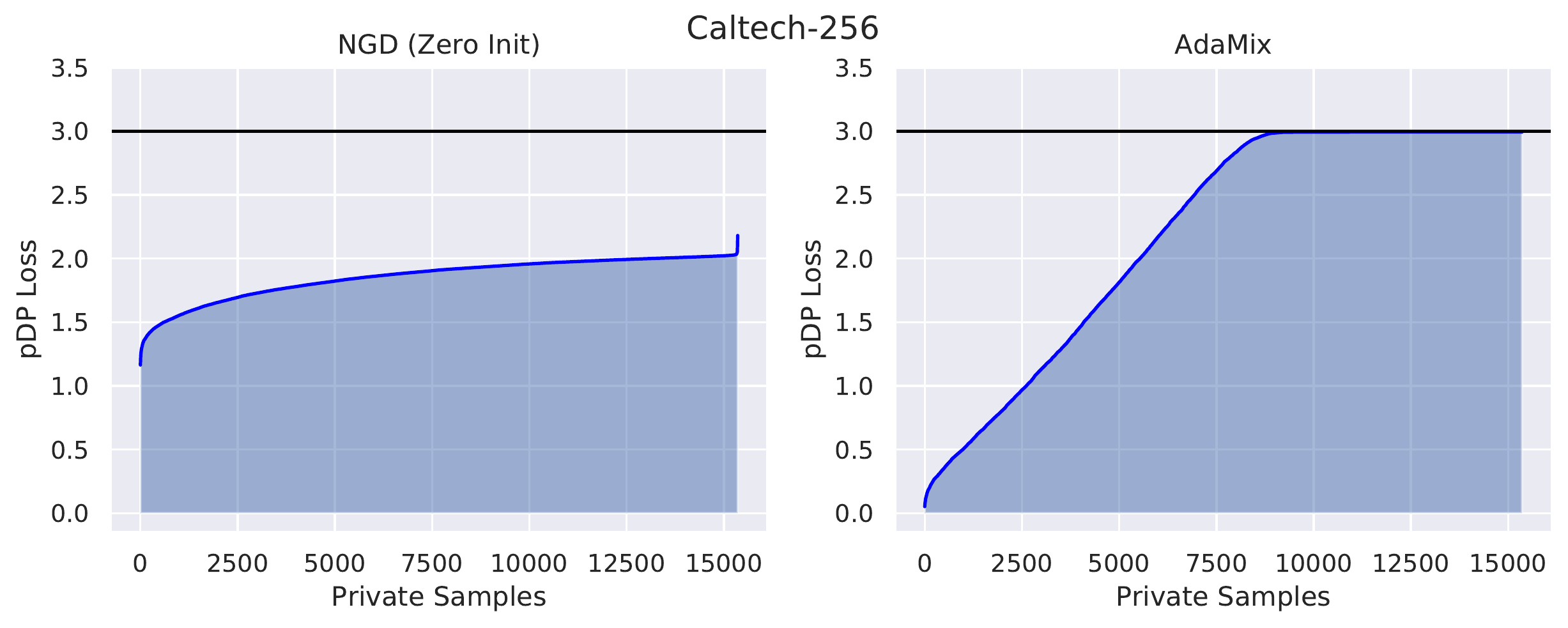}
\caption{Effect of public data on per-instance DP for Caltech-256}
\label{fig:pDP_caltech256}
\vspace{-1em}
\end{figure}

\subsection{Multi-modal Initialization}
\label{sec:multimodal_appendix}
We show the effect of using different initializations and CLIP features for different datasets in \cref{fig:multimodal_oxfordflowers,fig:multimodal_cub200,fig:multimodal_pets,fig:multimodal_stanforddogs}, similar to \cref{fig:multimodal} (we use $\epsilon=3$). We show that for Stanford Dogs and Oxford Pets datasets (\cref{fig:multimodal_pets} and \cref{fig:multimodal_stanforddogs}) which have images which are very similar to images in ImageNet, ResNet-50 trained on ImageNet performs better than CLIP features. However, the trend for the effect of initialization remains the same across all the datasets.

\begin{figure}[h]
\centering
\includegraphics[width=\linewidth]{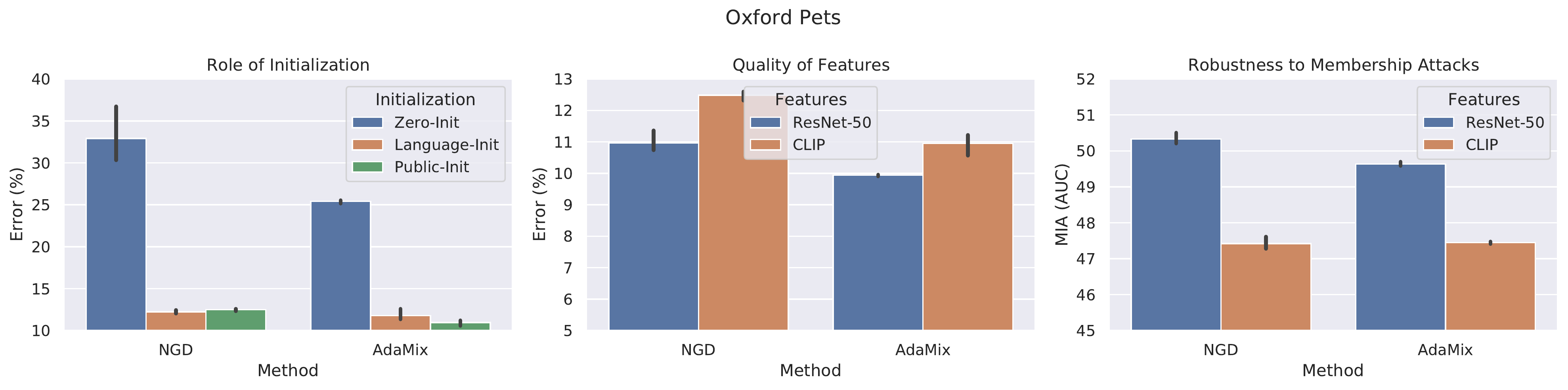}
\caption{Multi-model initialization and models for Oxford Pets}
\label{fig:multimodal_pets}
\end{figure}
\vspace{-0.5cm}

\begin{figure}[h]
\centering
\includegraphics[width=\linewidth]{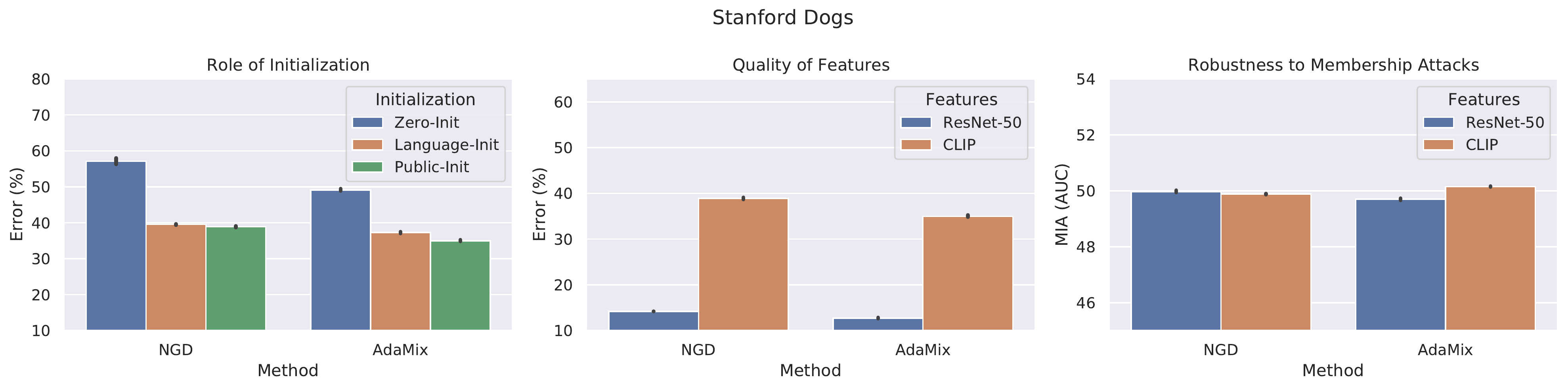}
\caption{Multi-model initialization and models for Stanford Dogs}
\label{fig:multimodal_stanforddogs}
\end{figure}
\vspace{-0.5cm}

\begin{figure}[h]
\centering
\includegraphics[width=\linewidth]{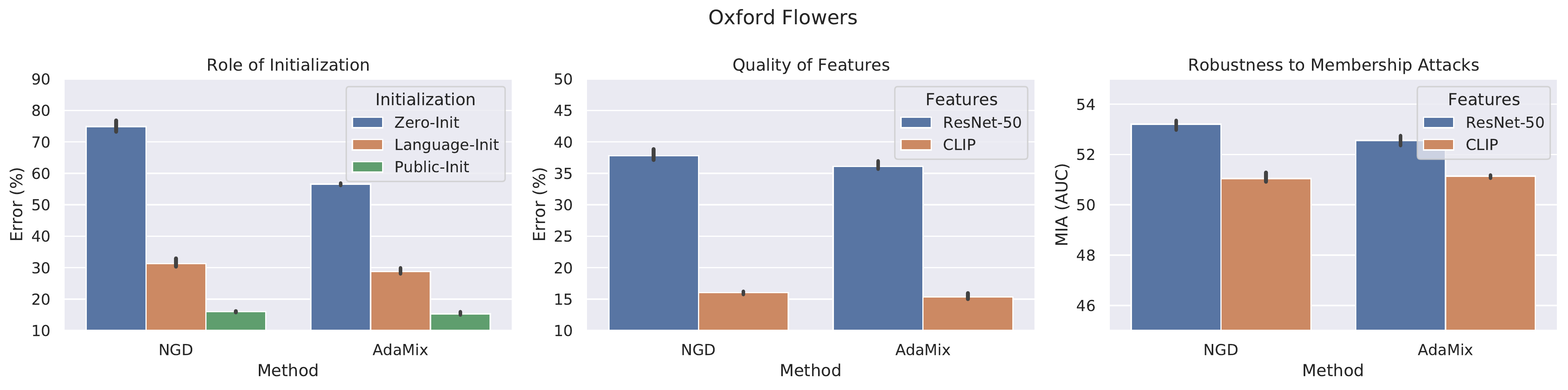}
\caption{Multi-model initialization and models for Oxford Flowers}
\label{fig:multimodal_oxfordflowers}
\end{figure}

\begin{figure}[h]
\centering
\includegraphics[width=\linewidth]{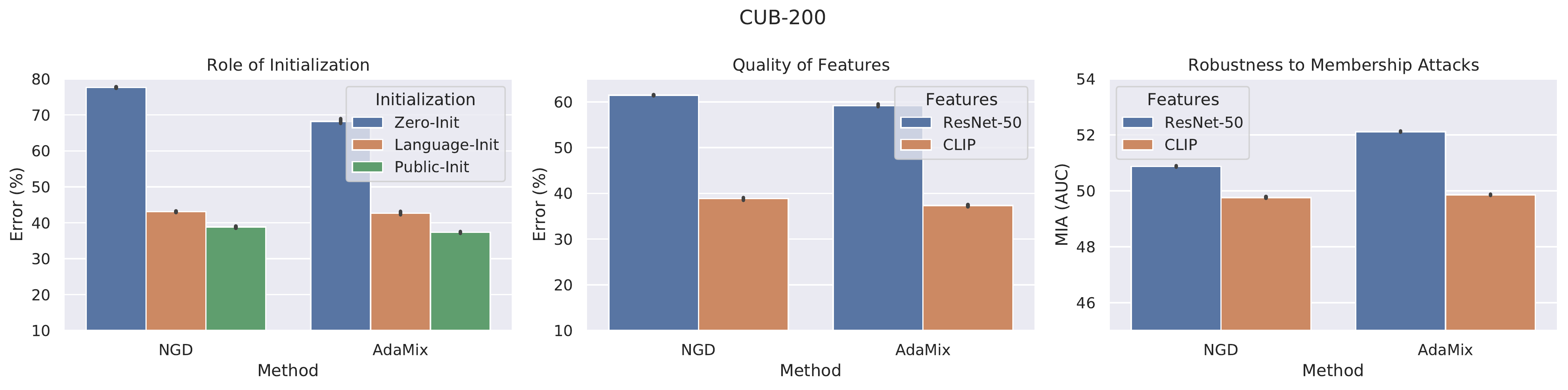}
\caption{Multi-model initialization and models for CUB-200}
\label{fig:multimodal_cub200}
\end{figure}

\subsection{Experimental Details}
\label{sec:exp_detail}
We use \href{https://github.com/yuxiangw/autodp}{Auto-DP} library for all of our experiments. For ResNet-50 features we use the \texttt{torchvision} version of the ResNet-50 model and for CLIP features we use the model provided.\footnote{\url{https://github.com/openai/CLIP}}

To create the public and private datasets, we take 2 samples (for Oxford-Flowers and CUB-200) and 5 samples (for MIT-67, Stanford Dogs, Oxford Pets, Caltech-256) from each class as the public dataset and the remaining samples as the private dataset. In this way, the public set is less than 10\% of the private set. We repeat all the experiments with 3 random seeds and report its mean and std.

For all the experiments we try multiple values for the learning rate: $\{1e-3, 2.5e-3, 5e-3\}$ for ResNet-50 experiments and $\{1e-6, 5e-7, 1e-7\}$ for CLIP experiments, and report the best values across 3 random seeds. We use $1e-2$ as the $L_2$ regularization coefficient across all the experiments. For subspace projection we project the gradients on a 2000-dimensional subspace for experiments with ResNet-50 features and 500-dimensional subspace for experiments with CLIP features. For the clipping threshold percentile we use a constant value of 90 percentile across all the experiments in the paper. In \cref{fig:effect-percentile}, we show that further tuning the clipping threshold percentile can improve performance on certain datasets, however, even with a pre-decided value of 90 percentile it still out-performs all the other methods.

We train the logistic model on the public data (few shot data) for 200 epochs (sames as iterations since we use gradient descent) for multiple values of learning rate: $\{1e-1, 5e-2, 1e-2\}$ for ResNet-50 features and $\{1e-4, 5e-5, 1e-5\}$ for CLIP features and choose the best performing model. For the private training, we use $\sigma=20$ for all the experiments and calculate the number of iterations required for private training using methods provided in \href{https://github.com/yuxiangw/autodp/blob/master/tutorials/tutorial_AdaSSP_vs_noisyGD.ipynb}{Auto-DP}. In the experiments we observe that choosing a higher value of $\sigma$ (thus training for more iterations) generally results in better performance.

\onecolumn
\clearpage
\section{Differential privacy basics}
\label{sec:dp-basics}
In this section, we describes tools we need from differential privacy and use them to prove that Algorithm~\ref{alg:mix-s} and Algorithm~\ref{alg:mix_theory} satisfies a family of differential privacy parameters.

\begin{lemma}[Analytical Gaussian mechanism {\cite{balle2018improving}}]
For a numeric query $f : \mathcal{X}^* \rightarrow \mathbb{R}^d$ over a dataset $\mathcal{D}$, the randomized algorithm that outputs $f(\mathcal{D}) + Z$ where $Z \sim \mathcal{N}(0, \sigma^2 I_d)$  satisfies $(\epsilon,\delta(\epsilon))$-DP for all $\epsilon\geq 0$ and
$\delta(\epsilon) = \Phi(\frac{\mu}{2}-\frac{\epsilon}{\mu}) - e^\epsilon \Phi(-\frac{\mu}{2}- \frac{\epsilon}{\mu})$
with parameter $\mu:=\Delta/\sigma$, where $\Delta:=\Delta_2^{(f)} = \max_{\mathcal{D} \sim \mathcal{D}^\prime} \| f(\mathcal{D}) - f(\mathcal{D}^\prime) \|_2$ is the global L2 sensitivity of $f$ and $\Phi$ is the CDF function of $\mathcal{N}(0,1)$.
\label{lm:ana_gau}
\end{lemma}
The above lemma tightly characterizes the $(\epsilon,\delta)$-DP guarantee of a single invocation of the Gaussian mechanism, and the following lemma shows that we can use the same result for an adaptive composition of a sequence of Gaussian mechanisms.

\begin{definition}[Gaussian Differential privacy {\cite{dong2021gaussian}}]
We say a mechanism $\cM$ satisfies $\mu$-Gaussian differential privacy (GDP), if it satisfies $(\epsilon,\delta(\epsilon))$-DP for all $\epsilon\geq 0$ and $\delta(\epsilon)$ being that of a single Gaussian mechanism (in Lemma~\ref{lm:ana_gau}) with parameter $\mu$. 
\end{definition}

\begin{lemma}[Composition of Gaussian mechanisms \cite{dong2021gaussian}]
The adaptive composition of a sequence of Gaussian mechanism with noise level $\sigma_1,\sigma_2,\dots$ and global L2 sensitivity $\Delta_1, \Delta_2, \dots$ satisfies $\mu$-GDP with 
parameter $\mu = \left(\sum_{i} (\Delta_{i}/\sigma_i)^2\right)^{1/2}$.
\label{lm:composition_gau}
\end{lemma}
Specifically, the noisy gradient descent (NoisyGD) algorithm (Algorithm~\ref{alg:noisyGD}) we use is a composition of $T$ Gaussian mechanisms for releasing the gradients and its privacy guarantee is equivalent to that of a single Gaussian mechanism.
\begin{proposition}
\label{prop:noisygd-is-private}
Let $\nabla f(w)$ be a function of the private dataset with global L2 sensitivity smaller than $L$ for any $w\in \cW$, then Algorithm~\ref{alg:noisyGD}
with parameter $T,\sigma^2$ such that $\rho:= \frac{T^2L^2}{2\sigma^2}$ satisfies $\sqrt{2\rho}$-Gaussian differential privacy.
\end{proposition}
\begin{proof}
The proof follows from Lemma~\ref{lm:composition_gau} as Algorithm~\ref{alg:noisyGD} is an adaptive composition of $T$ Gaussian mechanisms with global sensitivity $L$. %
\end{proof}

\section{Per-instance differential privacy}
\label{sec:perinstance-dp}
In this section, we provide details on the per-instance differential privacy \cite{wang2019per} that we used to generate Figure~\ref{fig:individual}. 
To cleanly define per-instance differential privacy, we first define indistinguishability.
\begin{definition}[$(\epsilon,\delta)$-indistinguishability]
We say two distributions $P,Q$ are $(\epsilon,\delta)$-indistinguishable if for any measurable set $S$
\begin{align*}
&\mathsf{Pr}_P[S] \le e^{\epsilon} \cdot \mathsf{Pr}_Q[S] + \delta
\text{ and } \mathsf{Pr}_Q[S] \le e^{\epsilon} \cdot \mathsf{Pr}_P[S] + \delta.
\end{align*}
\end{definition}

\begin{definition}[\cite{wang2019per}]
We say a randomized algorithm $\cM$ is $(\epsilon(\cdot),\delta)$-per-instance differentially private (pDP) for scalar $\delta\geq 0$ and function $\epsilon: \mathcal{Z}^* \times \mathcal{Z} \rightarrow \R_+$, such that for any dataset $D\in\cZ^*$, individual $z\in\cZ$,
$\cM(D)$ and $\cM(D\cup \{z\})$ are $(\epsilon(D,z),\delta)$-indistinguishable.
\end{definition}

pDP loss $\epsilon$ is a function of one particular pair of neighboring datasets. It describes the formal privacy loss incurred to the particular individual $z$ that is added (if $z$ is not part of the input dataset) or removed (if $z$ is part of the input dataset). 
pDP is a strict generalization of DP, as we can recover DP from pDP by maximizing $\epsilon(\cdot)$ over $D,z$.
\begin{lemma}
If $\cM$ is $(\epsilon(\cdot),\delta)$-pDP, then $\cM$ is also $(\sup_{D\in\cZ^*,z\in\cZ}\epsilon(D,z),\delta)$-DP.
\end{lemma}

We emphasize that the pDP loss $\epsilon(\cdot)$ itself is data-independent, but specific evaluations of the pDP losses (e.g., $\epsilon(D_{-z},z)$ or $\epsilon(D,z)$) depends on the private dataset $D$, thus should not be revealed unless additional care is taken to privately release these numbers. 

For our purpose, we are interested in the distribution of pDP losses of individuals in the dataset induced by different DP algorithms. This is used to provide a theoretically-sound alternative to the prevalent practices of using specific attacks (such as membership inference attacks) for evaluating the data-dependent privacy losses.
Before we state the pDP bounds of our algorithms, we extend the standard $(\epsilon(\cdot),\delta)$-pDP definition to a per-instance version of the Gaussian DP.
\begin{definition}
We say a mechanism is $\mu(\cdot)$-per-instance Gaussian Differentially Private (pGDP), if $(D,D\cup z)$ (and $(D\cup z, D)$) are $(\epsilon,\delta)$-indistinguishable for all $\epsilon,\delta$ parameters described by $\mu(D,z)$-GDP.
\end{definition}

This allows us to obtain precise pDP bounds under composition. 
\begin{proposition}[pDP analysis of AdaMix]\label{prop:pDP_of_adamix}
Let $z_1,...,z_n$ be the data points of the private dataset. Algorithm~\ref{alg:mix_theory} satisfies $\mu(\cdot)$-pGDP with
$$\mu(D_{-i},z_i) = \sqrt{\sum_{t=1}^{T} \frac{\min\{\|\nabla\ell_i(w_t)\|,\tau\}^2}{\sigma^2}}.$$
Similarly, Algorithm~\ref{alg:mix-s} satisfies $\mu(\cdot)$-pGDP with
$$
\mu(D_{-i},z_i) = \sqrt{\sum_{t=1}^{T} \frac{\|U^T \tilde{g}^\pri_i(w_t)\|_2^2}{\tau_t^2\sigma^2}}.
$$
\end{proposition}
\begin{proof}
    Both algorithms are the composition of $T$-Gaussian mechanisms. Thus the results follow by the composition of pGDP. The composition of pGDP is implied by the composition theorem of GDP (Lemma~\ref{lm:composition_gau}) by choosing the space of datasets to be just $\left\{ D,D\cup\{z\}\right\}$.
\end{proof}

Fixing any $\delta$, we can then compute the corresponding $\epsilon(\cdot)$ for $(\epsilon(\cdot),\delta)$-pDP using the formula of Gaussian mechanism with $\mu$ taken to be $\mu(\cdot)$ pointwise.

\section{Proofs of the technical results}
\label{sec:proofs}

\begin{algorithm}[t]
\caption{(Theoretical version of) AdaMix training algorithm (no clipping, no adaptive projection, slightly different pretraining, theoretically chosen learning rate).
}\label{alg:mix_theory}
\KwData{Public dataset $D_\pub$, private dataset $D_\pri$, privacy parameter $(\epsilon, \delta)$, noise variance $\sigma$, Lipschitz constant $L$\footnotemark, population-level strong convex parameter $c$, regularization parameter $\lambda$ and a constraint set $\cW$.}
\KwResult{$\bar{w}$}

\Comment{Public Pretraining Phase (OnePassSGD on $D_\pub$):}
$w_1 = 0$.
\For{$t = 1,\ldots,N_{\pub}$}{
\Comment{In a shuffled order}
$\eta_t = \frac{2}{c(t+1)}$\;
$w_{t+1} \gets \Pi_{\cW}\big(w_{t} - \eta_t \nabla \ell(w_t, (\tilde{x}_t,\tilde{y}_t))\big)$\;
}
$w_{\text{ref}} \gets w_{N_{\pub}+1}$.

\Comment{Mix Training Phase (NoisyGD on $D_\pub\cup D_\pri$):}

$T \gets \operatorname{Calibrate}(\epsilon, \delta, \sigma)$ \Comment{(i.e., $\frac{T\tau^2}{2\sigma^2} =: \rho$)}
\Comment{NoisyGD on objective function $\cL(w) + \frac{\lambda}{2}\|w-w_{\text{ref}}\|^2$}

$w_1 = w_{\text{ref}}$\;
\For{$t = 1,\ldots,T$}{
    $
    \eta_t \gets \frac{2}{\lambda(t+1)}
    $\;
    
    $n_t \sim N(0, \sigma^2 I)$\;
    
    $
    w_{t+1} \gets \Pi_{\cW}\big(w_t - \eta_t (\sum_{i=1}^{N_{\pri}} \nabla \ell_i(w_t) + \sum_{j=1}^{N_{\pub}} \nabla \tilde{\ell}_j(w_t) + n_t ) \big)
    $\;
}
\Comment{The following averaging can be implemented incrementally without saving $w_t$}
$\bar{w} \gets \sum_{t=1}^T\frac{2t}{T(T+1)}w_t$

\end{algorithm}
\footnotetext{As we discussed earlier, per-example gradient clipping can be viewed as Huberizing the loss function in GLMs \cite{song21evading}, all our results apply to the updated loss function. The adaptive clipping approach we took, can be viewed as an heuristic that automatically identifies an appropriate level of Huberization.}

We will be using the following $O(1/t)$ convergence bound due to Lacoste-Julien, Schmidt and Bach \cite{lacoste2012simpler}, which uses a decaying learning rate and a non-uniform averaging scheme. 
\begin{theorem}[Convergence of SGD for Strongly Convex Objectives {\cite{lacoste2012simpler}}]\label{thm:lacoste-julien}
Let $f$ be a $m$-strongly convex and defined on a convex set $\cW$. Assume stochastic gradient oracle $g_t$ satisfies that $\E[g_t | w_t] \in \partial f(w_t)$ and $\E[\|g_t\|^2] \leq G^2$ for all $t=1,...,T$.  Then the (projected) stochastic gradient descent with learning rate $\eta_t=\frac{2}{m(t+1)}$ satisfies
\begin{align}
    &\E[f(\sum_{t=1}^T \frac{2t}{T(T+1)}w_t)] - f(w^*)\leq \frac{2 G^2}{m(T+1)}\\
    \text{and }&\E[\|w_{T+1} - w^*\|^2] \leq \frac{4G^2}{m^2(T+1)}.
\end{align}
\end{theorem}

\begin{corollary}[NoisyGD for Strongly Convex Objectives]\label{cor:noisygd_bound}
Let $J(w) = \cL(w) + \frac{\lambda}{2}\|w\|^2$ with individual loss functions $\ell$ being $L$-Lipschitz on $\cW$.  Assume $\sup_{w\in\cW} \|w\| \leq B$. 
Let the learning rate be $\eta_t = \frac{2}{\lambda (t+1)}$, then 
		\begin{align}
		\E \left[J\left(\frac{2}{T(T+1)}\sum_{t=1}^T  t w_t\right)\right]  - J^* 
		\leq& \frac{2(N L + \lambda B)^2 }{\lambda T} + \frac{2d\sigma^2}{\lambda T}\\
		 =& \frac{2(N L + \lambda B)^2  }{\lambda T} + \frac{d L^2}{\lambda \rho},
	\end{align}
where $\rho:= \frac{TL^2}{2\sigma^2}$ is the privacy parameter of the algorithm ($\sqrt{2\rho}$-GDP).
\end{corollary}
\begin{proof}
First check that $NL+ \lambda B$ upper bounds the Lipschitz constant of $J$ on because the Lipschitz constant of $\frac{\lambda}{2}\|w\|^2$ is smaller than $\lambda B$ due to the bounded domain.
Second, check that the noisy gradient oracle satisfies that it is unbiased, and the added noise has a variance of $\sigma^2$ per coordinate for all $d$ coordinates. Thus 
$$
\E[\|g_t\|^2 | w_t] = \E[\|\nabla J(w_t)\|^2 | w_t] + \E[\|n_t\|^2|w_t] \leq (NL + \lambda B)^2 + d\sigma^2. 
$$
Thus by taking expectation on both sides we verify that we can take $G^2 = (NL + \lambda B)^2 + d\sigma^2$.

It remains to substitute these quantities and apply the first statement of Theorem~\ref{thm:lacoste-julien}.
\end{proof}

\begin{corollary}[One-Pass SGD on public data]\label{cor:sgd_public_data}
Assume the public data with $N$ samples are drawn from the same distribution of the private data. Assume that the (population risk)
$$
R(\theta) = \E_{(x,y)\sim \cD}[ \ell(\theta^*, (x,y))]
$$
is $c$-strongly convex at $\theta^*$ for some constant $c$. Then the one-pass SGD algorithm below
$$
w_{t+1} = w_t - \frac{2}{c(t+1)} \nabla \ell(w_t, (x_t,y_t))
$$
for $t=1,...,N_\pub$
obeys that 
	$$
	\E[\|w_{N_\pub +1} - w^*\|^2] \leq \frac{4L^2}{ c^2 N_\pub}
	$$
	where $w^* = \argmin_{w\in \cW} \cR(w).$
\end{corollary}
\begin{proof}
First note that since the  data is drawn iid, running one-pass SGD by going through the data points in a random order uses a \emph{fresh sample} to update the parameters. This is is equivalent to optimizing the population risk directly. Check that for any fixed $w$ and all $t=1,...,N_\pub$ $\E_{(x_t,y_t)\sim \cD}[\nabla_{w} \ell(w,(x_t,y_t))] = \nabla \cR(w).$ 
Moreover, we need this stochastic gradient oracle to satisfy
$\E_{(x,y)\sim \cD}[\|\nabla \ell(\theta,(x,y))\|^2]\leq  G^2$. Notice that by our assumption $\|\nabla \ell(\theta,(x,y))\|^2\leq L^2$, thus we can take $G=L$.
By invoking the second statement of Theorem~\ref{thm:lacoste-julien} the result follows.
\end{proof}

With these two corollaries stated, we are now ready to prove Theorem~\ref{thm:excess_empirical_risk_bound} and Theorem~\ref{thm:provable_benefit_of_public_data}.

\begin{proof}[Proof of Theorem~\ref{thm:excess_empirical_risk_bound}]
    The proof relies on Corollary~\ref{cor:noisygd_bound}, and the following argument.
    When additional regularization with parameter $\lambda $, the utility we consider should still be considered in terms of $\cL(\hat{w}) - \cL(w^*)$. Let $w^*$ be any comparator satisfying $B> \|w^*\|$
\begin{align*}
	 &\cL(\bar{w}) - \cL(w^*) 
	 =J(\bar{w}) - J_\lambda(w^*_\lambda)  \\
	 &+J_\lambda(w^*_\lambda) - J(w^*) +  \frac{\lambda}{2}\|w^*-w_{\text{ref}}\|^2 -   \frac{\lambda}{2}\|\hat{w}-w_{\text{ref}}\|^2\\
	 \leq& J(\hat{w})  - J_\lambda(w^*_\lambda) + \frac{\lambda}{2}\|w^*-w_{\text{ref}}\|^2.
\end{align*}
Take expectation on both sides and apply Corollary~\ref{cor:noisygd_bound}
{\small
\begin{align*}
	\E[\cL(\bar{w})]- \cL(w^*) \leq \frac{2(N L + \lambda B)^2  }{\lambda T} + \frac{d L^2}{\lambda\rho} +  \frac{\lambda}{2}\|w^*-w_{\text{ref}}\|^2.
\end{align*} 
}
Finally, choosing $\lambda  = \sqrt{\frac{\rho}{2\|w^*-w_{\text{ref}}\|d L^2}}$ yields 
$$
\E[\cL(\hat{w})]- \cL(w^*) \leq   \frac{4(n L + \lambda B)^2  }{\lambda T}  + \frac{\sqrt{d} L\|w^*-w_{\text{ref}}\|}{\sqrt{2\rho}}
$$
as claimed (dividing $N$ on both sides to get $\hat{R}$).
\end{proof}

\begin{theorem}\label{thm:provable_benefit_of_public_data}
	Assume the private data and public data are drawn i.i.d.\ from the same distribution $\cD$ and that $R(w) = \E_{(x,y)\sim \cD}[ \ell(w, (x,y))]$ is $c$-strongly convex in $\cW$.  Let $w_{\mathrm{ref}} = w_{N_\pub +1}$ --- the last iterate of a single pass stochastic gradient descent on $\hat{\cR}_\pub (w)$ (initializing from $w_0= 0$) that goes over the public dataset exactly once one data-point at a time with learning rate $\eta_t = \frac{2}{c(t+1)}$. Let  $w_{\mathrm{ref}}$ be passed into Theorem~\ref{thm:excess_empirical_risk_bound}'s instantiation of NoisyGD, which returns $\bar{w}$ (The pseudo-code of this algorithm is summarized in Algorithm~\ref{alg:mix_theory} in the appendix), then at the limit when $T$ is sufficiently large, the \emph{excess risk} obeys that
	\begin{align*}
	&\E[\cR(\bar{w})]  -  \cR(w^*) \\
	\leq& \frac{4\sqrt{d} L^2}{c\sqrt{N_\pub}N\sqrt{2\rho} }+  \text{Gen}(\bar{w},N) + \text{Gen}(w^*,N),\end{align*}
	where $w^* = \argmin_{w\in\cW} \cR(w)$ and 
	$
	\text{Gen}(w,N):= \left|\E[\cR(w) - \hat{\cR}(w)]\right|
	$ is the expected generalization gap of (a potentially data-dependent) $w$.
\end{theorem}

\begin{proof}[Proof of Theorem~\ref{thm:provable_benefit_of_public_data}]
	Let $w^* = \arg\min_{w\in \cW} \cR(w)$. By Corollary~\ref{cor:sgd_public_data}, we have 
	$$
	\E[\|w_{\text{ref}} - w^*\|^2] \leq \frac{4L^2}{ c^2 N_\pub},
	$$
	which implies (by Jensen's inequality) that 
	$\E[\|w_{\text{ref}} - w^*\|] \leq \sqrt{\frac{4L^2}{ c^2 N_\pub}}.$
	
Now by plugging in the $w^*$ in theorem,  take expectation over the public dataset, and substitute the above bound, we get
$$
	\E[\E[\hat{\cR}(\bar{w})]- \hat{\cR}(w^*) ]\leq   \frac{4(N L + \lambda B)^2  }{\lambda TN}  + \frac{2\sqrt{d} L^2}{c\sqrt{N_\pub}N\sqrt{2\rho} }.
$$
Take $T$ to be sufficiently large so that the second term dominates, we obtain the stated bound. 

Finally, to convert the above bound into that of  the excess risk:
\begin{align*}
	&\E[\E[\hat{\cR}(\bar{w})]- \hat{\cR}(w^*) ] -  (\E[\cR(\bar{w})]- \cR(w^*) )\\
	\leq&  |\E[\E[\hat{\cR}(\bar{w})] - \cR(\bar{w}) ] | +  |\E[\hat{\cR}(w^*) ] - \cR(w^*)|\\
	 :=& \text{Gen}(\bar{w},N)+  \text{Gen}(w^*,N),
\end{align*}
which completes the proof.
\end{proof}

We make two additional remarks. First, we do not require the \emph{empirical} objective $\cL_\pub$ to be strongly convex. In practice, we do not have strong convexity when $N_\pub < d$. The assumption of $c$-strong convexity is on the \emph{population-level} risk function $\cR$.
Second, our bound decomposes the excess (population) risk of the private learner into a (local) uniform-convergence bound (which is required by a non-private learner too) and an additional cost due to privacy. Note that $\text{Gen}(N)$ is usually $O(1/\sqrt{N})$ but could be $O(1/N)$ when certain data-dependent ``fast rate'' conditions are met, e.g., realizability, low-noise, or curvature (see, e.g., \cite{koren2015fast}).
Our results suggest that the cost of privacy asymptotically vanishes (fix $\rho$, $N_\pub\rightarrow \infty$, and $N_{\pub}/N\rightarrow 0$) even under these fast rate conditions relative to the non-private rate.

\end{document}